\documentclass[11pt]{article}

\usepackage[margin=1in]{geometry}
\usepackage[numbers,sort&compress]{natbib}
\usepackage{microtype}
\usepackage{amsmath,amsfonts,amssymb}
\usepackage{amsthm}
\usepackage{graphicx}
\usepackage{subfigure}
\usepackage{booktabs}
\usepackage{caption}
\usepackage{multirow}
\usepackage{wrapfig}
\usepackage{float}
\usepackage{xspace}
\usepackage{comment}
\usepackage[inline]{enumitem}
\setlist[enumerate]{itemsep=0pt,topsep=0pt,leftmargin=*}

\usepackage{amsmath,amsfonts,bm}

\newcommand{\captiona}{{\em (a)}}
\newcommand{\captionb}{{\em (b)}}
\newcommand{\captionc}{{\em (c)}}

\def\eqref#1{equation~\ref{#1}}

\def\1{\bm{1}}

\DeclareMathAlphabet{\mathsfit}{\encodingdefault}{\sfdefault}{m}{sl}
\SetMathAlphabet{\mathsfit}{bold}{\encodingdefault}{\sfdefault}{bx}{n}

\usepackage{algorithm}
\usepackage{algpseudocode}
\newtheorem{proposition}{Proposition}
\usepackage[dvipsnames]{xcolor}
\usepackage{hyperref}
\makeatletter
\providecommand{\theHALG@line}{\thealgorithm.\arabic{ALG@line}}
\makeatother
\usepackage[capitalise]{cleveref}
\usepackage[compact]{titlesec}
\title{FoldAttention: Declared-Reference Softmax for Fast Decode and Deterministic Backward}

\author{%
  \begin{tabular}{c}
    Sriman Achanta\textsuperscript{1}
  \end{tabular}
  \\
  {\small
    \textsuperscript{1}Virginia Commonwealth University
  }
  \\
  \texttt{achantass@vcu.edu}
}

\date{}

\begin{document}

\maketitle

\begin{abstract}
Autoregressive decode repeatedly streams a growing KV cache, making attention
a major cost at long context. Existing high-performance kernels use online
softmax, which discovers a row's normalization reference as it scans keys.
Earlier contributions therefore remain provisional and may require rescaling.
We argue that the reference need not be discovered: softmax is invariant to a
common shift, so the reference only has to keep the weights in range. We
present FoldAttention, an additive formulation of softmax attention that
fixes a finite reference $Z_i$ before scanning the KV cache. Each weight
$2^{s_{ij}-Z_i}$ is then final when computed, so contributions add across
disjoint key ranges and their quotient equals softmax attention in real
arithmetic. We use this property to develop two techniques for Hopper decode:
(1) final weights gate key and value reads before the bytes are fetched, and
a per-call depth $T$ cuts keys below $2^{-T}$ while keeping their mass, and
(2) additive partials compose split KV and shared-prefix cascades without
rescaling. On H100 at $T=16$, FoldAttention decodes seven real-model
generations 1.36--2.30$\times$ faster than the fastest BF16 baseline, and up
to 3.09$\times$ faster across MHA and GQA shapes, at an error within 1.5\% of
the lowest BF16 error on six of the seven; reading every key, it is
1.14--1.30$\times$ faster at matched error. We validate on Qwen3-8B that a
whole decode step is up to 1.46$\times$ faster while likelihood and
long-context accuracy match those under BF16 kernels. The same principle
makes the backward deterministic: CTAs round bounded partial gradients onto
an integer grid declared before the reduction and add them in any order.
FoldAttention thereby removes the determinism tax: its deterministic backward
is up to 1.84$\times$ faster than deterministic FlashAttention-3/4 and
1.05$\times$ faster than the fastest nondeterministic kernel.
\end{abstract}

\section{Introduction}
\label{sec:intro}

\begin{figure}[t]
\centering
\includegraphics[width=\linewidth]{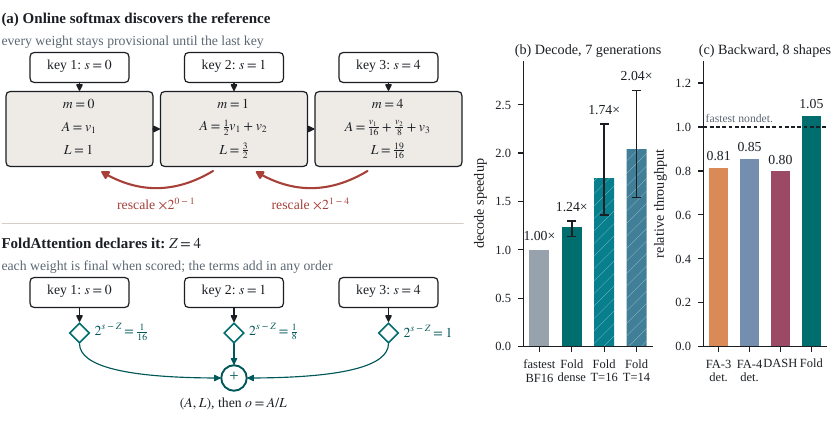}
\caption{\captiona\ Online softmax rescales $A$ and $L$ each time the
running maximum rises; with $Z=4$ fixed first, each weight is final and the
terms add in any order. \captionb\ Decode speedup over the fastest BF16
kernel on seven generations (geometric mean and range). \captionc\
Deterministic backward throughput relative to the fastest nondeterministic
kernel, over eight shapes.}
\label{fig:hero}
\end{figure}

Autoregressive decode reads the KV cache to produce one token per request. It
is therefore bound by memory bandwidth, and its cost grows with context and
batch size. At long context, attention alone takes up to half of a decode
step~\citep{yi2026pat}. Training has a separate bottleneck in the attention
backward. Many thread blocks (CTAs) contribute to each gradient, and their
arrival order determines the result's bits. Reproducing a training run
requires those bits to remain fixed~\citep{qiang2026dash,he2025nondeterminism}.

High-performance kernels, FlashAttention and its successors among them, use
online
softmax~\citep{milakov2018online,dao2022flashattention,dao2023flashattention2,shah2024flashattention3,zadouri2026flashattention4}.
Each key tile is weighted relative to a running maximum, and a later increase
rescales the accumulated state.
No weight is final until the last key is scored. Measured against the running
maximum, which each cache split restarts, a weight only bounds its final
value, so reads can be skipped only conservatively and split partials merge
only by comparing maxima and rescaling. Backward reductions have an analogous
dependency: FP32 addition rounds relative to the running total, so the arrival
order of CTAs determines the bits. Deterministic modes pay a
\emph{determinism tax} to impose an order: up to 38\% of throughput in
FlashAttention-3 and 25\% in
FlashAttention-4~\citep{qiang2026dash,zadouri2026flashattention4}.

Both dependencies arise because a numerical scale is discovered during
parallel work. FoldAttention instead declares the scale before that work
begins. For decode, softmax invariance means that the reference only has to
keep weights inside the exponent range of FP32 accumulation and BF16
tensor-core operands. A cheap per-row estimate lands within 29 binades of the
true log-sum-exp on every row we measured, up to 128K keys. Prior work fixes
or freezes a reference to reduce
rescaling~\citep{hong2024flashdecodingpp,zadouri2026flashattention4,sun2026vfa},
or drops blocks below a pseudo-maximum together with their
mass~\citep{liu2026ffd}. FoldAttention uses a declared reference more broadly:
splits share it, final weights gate reads, and cut keys still contribute to
the denominator. For backward, every CTA derives the same declared rounding
grid, so partial gradients add as integers in any order.

We present FoldAttention, which fixes a finite reference $Z_i$ for each
query row before any key is scanned (Figure~\ref{fig:hero}a). We contribute:
\begin{enumerate}
\item \textbf{One additive contract} (Section~\ref{sec:contract}) for
split-KV and cascades, exact in real arithmetic with finite-precision range
bounds, and an order-free integer reduction for backward.
\item \textbf{A decode kernel that reads by final weight}
(Section~\ref{sec:decode}). Keys are two INT8 planes with a per-key scale;
the first plane gives each key's final weight, which decides whether the
second plane and the value row are read. A depth $T$, set per call on the
same cache, cuts keys below $2^{-T}$ while keeping their mass, a dial from
BF16-kernel accuracy to fewer bytes. Shared prefixes are read once for every
request that holds them.
\item \textbf{A deterministic backward at nondeterministic speed}
(Section~\ref{sec:backward}). CTAs round their partials onto a common
power-of-two grid, folded for $dQ$ into the softmax's own exponential, and
add integers in any order.
\end{enumerate}

Against tuned FlashAttention-3/4, FlashInfer, TensorRT-LLM
XQA~\citep{nvidia2026tensorrtllm}, cuDNN~\citep{nvidia2026cudnn}, and
DASH~\citep{qiang2026dash} on an H100, scored against one FP32 reference,
FoldAttention at depth 16 decodes seven multi-step generations of
Qwen3-30B-A3B~\citep{yang2025qwen3}, gpt-oss-20b~\citep{openai2025gptoss},
and GLM-4-9B~\citep{glm2024chatglm} 1.36--2.30$\times$ faster than the
fastest BF16 kernel. Its output error is within 1.5\% of the lowest BF16
error in six cases, and it reaches 3.09$\times$ on a 1K--32K sweep of MHA
and GQA shapes. On Qwen3-8B, a whole decode step is 1.13--1.46$\times$
faster, and likelihood, retrieval, and LongBench accuracy match those under
BF16 kernels from 8K to 128K context. The deterministic backward is up to 1.84$\times$ faster
than deterministic
FlashAttention-3/4 and 1.05$\times$ faster than the fastest nondeterministic
kernel, with bit-identical gradients across runs, batches, and packings.

We open source FoldAttention with a permissive Apache-2.0 license. The code is available at
\href{https://github.com/srimanachanta/fold-attention}{https://github.com/srimanachanta/fold-attention}.
\section{Background}
\label{sec:background}

\paragraph{Online softmax and its merge.}
Let $q_i$ be a query row, $k_j,v_j$ the keys and values, and
$s_{ij}=\log_2(e)\,q_i^\top k_j/\sqrt{D}$ the score in base 2, so
attention~\citep{vaswani2017attention} returns
$o_i=\sum_j 2^{s_{ij}}v_j\big/\sum_j 2^{s_{ij}}$. Online
softmax~\citep{milakov2018online} keeps a running maximum $m$, numerator $A$,
and denominator $L$, and when a tile raises $m$ to $m'$ multiplies $A$ and
$L$ by $2^{m-m'}$ before adding the tile's terms $2^{s_{ij}-m'}$. Split-KV
decode~\citep{dao2023flashdecoding} returns a partial state $(o,\ell)$ per
split, with $\ell=\log L+m$, and merges two as
\begin{equation}
(o,\ell)\oplus(o',\ell')=
\Bigl(\tfrac{e^{\ell}o+e^{\ell'}o'}{e^{\ell}+e^{\ell'}},\;
\log\bigl(e^{\ell}+e^{\ell'}\bigr)\Bigr),
\label{eq:lse-merge}
\end{equation}
as FlashInfer, LeanAttention, and Hydragen do for split-KV, cascades, and
shared prefixes~\citep{ye2025flashinfer,sanovar2024leanattention,juravsky2024hydragen}.
The merge is evaluated stably only by comparing $\ell$ and $\ell'$ and
rescaling one side, the same step the scan takes.

\paragraph{Decode traffic and gradient reduction.}
A BF16 decode step reads 512 bytes of key and value per cached token at
$D=128$. The
fastest BF16 kernel already reaches 93--99\% of the H100's measured bandwidth
at 16K keys and beyond, so a faster decode must read fewer bytes. FP8
caches~\citep{micikevicius2022fp8} halve this traffic but raise the output's
FP32-relative error by 27--133$\times$ in our benchmarks. In backward, with
$D_i=dO_i^\top o_i$ and $dS=P\circ(dO\,V^\top-D)$, each CTA holds a key block
and adds a partial into every row of $dQ=dS\,K$ it touches. FP32 addition is
not associative, so with atomic adds the bits of $dQ$ follow the CTAs'
finishing order, and deterministic modes serialize the
adds~\citep{qiang2026dash,zadouri2026flashattention4}.

\section{The additive contract}
\label{sec:contract}

\paragraph{Fixed-reference softmax.}
Fix a finite reference $Z_i$ for query row $i$ before any key is read. For a
set of keys $I$ define
\begin{equation}
F_i(I)=(A_i^I,L_i^I)=\Bigl(\sum_{j\in I}2^{s_{ij}-Z_i}v_j,\;
\sum_{j\in I}2^{s_{ij}-Z_i}\Bigr),\qquad
o_i=A_i^I/L_i^I .
\label{eq:fold}
\end{equation}
Every term depends on one score and on $Z_i$, and on nothing seen before or
after it.

\begin{proposition}[Exact additive softmax]
\label{prop:fold}
Let $Z_i$ be finite. \textup{(i)} For disjoint key sets,
$F_i(I)+F_i(J)=F_i(I\cup J)$. \textup{(ii)} $A_i/L_i$ equals the softmax
attention output of row $i$ over the represented scores and values, for
every $Z_i$. \textup{(iii)} Let the weights be evaluated and accumulated in
FP32 with flush to zero, over $n\le2^{22}$ keys with $m_i=\max_j s_{ij}$
and $\nu=\max(1,\max_{j,d}|v_{jd}|)$. If
$m_i-Z_i+\log_2(n\nu)<127$, no weight, numerator, or denominator overflows,
and flushed weights change $L_i$ by a relative amount of at most
$n\,2^{-126-(m_i-Z_i)}$.
\end{proposition}

Appendix~\ref{app:proofs} gives the proof. Where online softmax merges
partial states with Equation~\ref{eq:lse-merge}, these pairs merge with $+$:
no maximum is compared and no finished term is rescaled, so split-KV,
shared-prefix cascades, and warp completion order are all groupings of one
sum. The decode kernel applies the identity to inputs its gates choose
(Section~\ref{sec:decode}), and Section~\ref{sec:evaluation} measures its
error against FP32.

\paragraph{Choosing the reference.}
Part (iii) is the only condition $Z_i$ must meet. The BF16 operand that
carries each weight into the tensor core has FP32's exponent range, so the
condition is the same for the product. $Z_i$ need not be the maximum or a
bound on it. An estimate too high or too low scales every weight by the same
factor, which cancels in $A/L$, so only range matters: with $n\le2^{22}$ keys
and values below $2^{16}$, any $Z_i$ within about 80 binades of the row's
largest score keeps every weight and sum in range and loses less than
$2^{-24}$ of $L_i$ to flushing.

FoldAttention estimates $Z_i$ per row and step as a log-sum-exp from 192
keys scored with the kernel's own coarse logits: the attention
sink~\citep{xiao2024streamingllm} and the 63 most recent keys exactly, plus a
trimmed stratified estimate from 128 stratum centers over the rest, with no
state carried between steps. The true log-sum-exp sits $-0.4$ to $17.2$
binades above $Z_i$ over every row of the seven generations of
Section~\ref{sec:evaluation}, and $-0.9$ to $28.2$ in Qwen3-8B at 64K--128K
context: more than an FP16 weight holds, which is why a fixed FP16 reference
needs a fallback~\citep{hong2024flashdecodingpp}, and far inside BF16's range
(Figure~\ref{fig:headroom}; Appendix~\ref{app:decode}). A cache can also certify the range after each
step: a row whose $L_i=2^{\ell_i-Z_i}$, with $\ell_i$ its base-2
log-sum-exp, leaves $[2^{-1},2^{100}]$, or whose output is not finite, is
decoded again with $Z_i=\ell_i$. The check syncs with the host, so timed runs
omit it; across 1.8 billion rows decoded in Qwen3-8B it would rerun 118,
each with $Z_i$ less than 1.6 binades above $\ell_i$.

\begin{figure}[t]
\centering
\includegraphics[width=\linewidth]{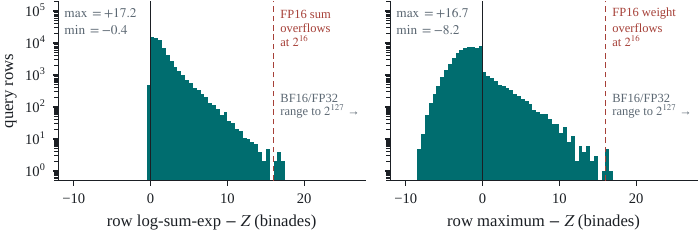}
\caption{How far each row's true log-sum-exp (left) and largest score
(right) sit above the estimated $Z_i$, over 63{,}488 rows of the seven
generations. FP16 weights would overflow at the right edge; BF16 and FP32
extend to $+127$.}
\label{fig:headroom}
\end{figure}

\paragraph{Integer-grid gradients.}
The backward needs the same declaration one level down: a rounding scale
fixed before the parallel reduction, so arrival order cannot change the sum.
It adds partials $g_c$ from $C$ CTAs into one gradient element.
Given a bound $\sum_c|g_c|\le B$, choose the power of two
$\alpha=2^{b-\lceil\log_2 B\rceil}$ and sum integers:
\begin{equation}
\hat g_c=\operatorname{RNE}(\alpha g_c),\qquad
g=\alpha^{-1}\textstyle\sum_c \hat g_c .
\label{eq:grid}
\end{equation}

\begin{proposition}[Order-free gradient sums]
\label{prop:grid}
If $C<2^{b+1}$, the $\hat g_c$ summed in $(b+2)$-bit two's-complement
arithmetic give the same $g$ in every order and grouping, no partial sum
overflows, and $|g-\sum_c g_c|\le C\,B\,2^{-b}$.
\end{proposition}

This is pre-rounding in the sense of reproducible
summation~\citep{demmel2013fast,ahrens2020reprotoms}; what attention adds is
cheap bounds. With $P$ a row-stochastic matrix, $\|v_j\|\le C_D\max|V|$ and
Cauchy--Schwarz, the terms of one $dQ$ element satisfy
\begin{equation}
\textstyle\sum_j|dS_{ij}K_{jd}|\le\bigl(\max_i\|dO_i\|\,C_D\max|V|+\max_i|D_i|\bigr)\max|K|,
\label{eq:dq-bound}
\end{equation}
where $C_D$ is $\sqrt D$ rounded up to a power of two and the maxima run over
one request's rows and one KV head. Every CTA's partial is a subset of these
terms, so the right side bounds $\sum_c|g_c|$. The kernel widens it by
$2^{-6}$ for the BF16 rounding of $P$ and $dS$. A preprocess reduces the
maxima with order-free integer maxima over their bit patterns, and every CTA
evaluates the bound with the same instructions, so all arrive at the same
grid: like $Z_i$, it is declared before the reduction, not discovered by it.

\section{Decode kernel}
\label{sec:decode}

FoldAttention uses a paged cache~\citep{kwon2023pagedattention}, a per-step
front kernel that writes the new token and estimates $Z_i$, and a split-KV
kernel. The fixed reference governs every read decision in the split-KV
kernel.
Figure~\ref{fig:decode-kernel} shows one CTA's work on one tile.

\begin{figure}[t]
\centering
\includegraphics[width=\linewidth]{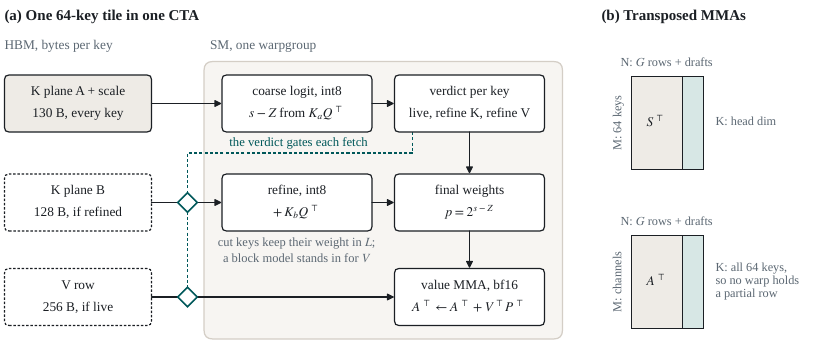}
\caption{\captiona\ One CTA on a 64-key tile, bytes per key at $D=128$: plane
A and the key's scale give its score against $Z_i$, and the verdict, ORed
over the GQA group, gates the reads of plane B and the value row.
\captionb\ Keys or channels sit on $M$ and the group's query rows (and any
draft rows) on $N$, so the value product reduces over a whole tile.}
\label{fig:decode-kernel}
\end{figure}

\paragraph{Cache format.}
Keys and queries are rotated by an orthonormal Hadamard matrix, which keeps
every $q^\top k$ and spreads outlier channels~\citep{ashkboos2024quarot}.
Each key is stored as two INT8 planes, $k_j=e_j(a_j+b_j/256)$, with a BF16
scale $e_j$ per key rounded up so plane A never saturates. A key's code
therefore depends on that key alone: no calibration pass, no headroom, and
appended tokens are quantized as precisely as the prompt. The query is split
the same way each step. Values are BF16, or two
E4M3~\citep{micikevicius2022fp8} planes under a power-of-two scale per layer.
The kernel widens the E4M3 values to BF16 exactly by bit placement. Plane A
and its scale cost 130 bytes per key at $D=128$.

\paragraph{Weight-gated reads.}
Each step, a front kernel quantizes the query, appends the new token, and
estimates every row's $Z_i$ (Section~\ref{sec:contract}). Each decode CTA is
one warpgroup and walks its split in tiles of 64 keys; a request's splits
interleave by tile, because live keys cluster in the most
recent tiles and contiguous splits would leave them all to one CTA. An
INT8 WGMMA scores every key against every row of the GQA~\citep{ainslie2023gqa} group, keys on $M$
and rows on $N$ (Figure~\ref{fig:decode-kernel}b). Against the fixed $Z_i$,
each score yields three final verdict bits: \emph{live} ($s\ge Z_i-T$),
\emph{refine K}, and \emph{refine V}, set when a weight is large enough for
the second plane to matter at the error target; the refine thresholds follow
each request's own length, so its gates do not depend on its batch. The bits
are ORed over the group, and only then are bytes requested: plane B for
refined keys, added through a second INT8 product, and value rows for live
keys. A tile with no live key skips its value work. Against a running
maximum these reads could be skipped only conservatively
(Section~\ref{sec:eval-ablation}). At $D=64$, groups of up to four rows walk
128-key tiles that hold two keys in each row of the product
(Appendix~\ref{app:tiles}).

\paragraph{Value product.}
Weights $p=2^{s-Z_i}$ are formed in FP32 and fed to a BF16 WGMMA,
$A^\top\leftarrow A^\top+V^\top P^\top$, with channels on $M$, rows on $N$,
and the tile's 64 keys on the reduction dimension, so no warp holds a partial
row. For dense decode and depths of 15 or more, the BF16 rounding residual
of $p$ rides a second product; $L$ sums the weights the products use.
Each split writes an FP32 pair $(A,L)$, and a combine adds the pairs in a
fixed slot order (Algorithm~\ref{alg:decode}). Speculative
drafts~\citep{leviathan2023fast} and verification
trees~\citep{miao2024specinfer} are extra query columns masked to their
ancestors. Because pairs add, a prefix shared by many requests is decoded once
with their query rows stacked, and its $(A,L)$ is added to each suffix's in
the combine (Appendix~\ref{app:tiles}).

\paragraph{Depth as a dial.}
With depth $T$, cut keys add their weight to $L$ but not their value to $A$.
$T$ is a parameter of each call, not of the cache, so one cache serves every
depth, and $T=\infty$ is dense decode. For BF16 values the kernel substitutes
a model of each 64-key block, its mean value plus a rank-16 map from key to
value deviation fitted on the prompt, and one virtual row per block carries
the cut keys' summed weight; for 8-bit values the model is the request's
running mean, which needs no fit. We evaluate two finite settings. At $T=16$,
the error is within 1.5\% of the lowest BF16 error on six of the seven
generations in Section~\ref{sec:evaluation}. At $T=14$, the kernel reads fewer
bytes while remaining within the BF16 error range on those generations.

At $D=128$ a BF16 key and value cost 512 bytes. The dense path reads
$130+128\,r_K+256$ bytes per key, where $r_K$ is the refined fraction; 8-bit
values replace 256 with $128+128\,r_V$. At finite depth, BF16 value traffic
is $256\,\lambda$ for live fraction $\lambda$. On the ablation cells of
Section~\ref{sec:eval-ablation} these measure 389--423, 261--280, and
196--334 bytes per key. Every configuration stores both planes, 514 bytes per
key; the gates save reads.

\section{Deterministic backward kernel}
\label{sec:backward}

Training pairs FlashAttention-4's forward with a FoldAttention backward that
keeps the structure of FlashAttention-3's Hopper kernel (a TMA producer warp,
two ping-ponged MMA warpgroups, one key block per CTA) and changes how CTAs'
partial gradients are added (Figure~\ref{fig:backward-reduction}).

\begin{figure}[H]
\centering
\includegraphics[width=\linewidth]{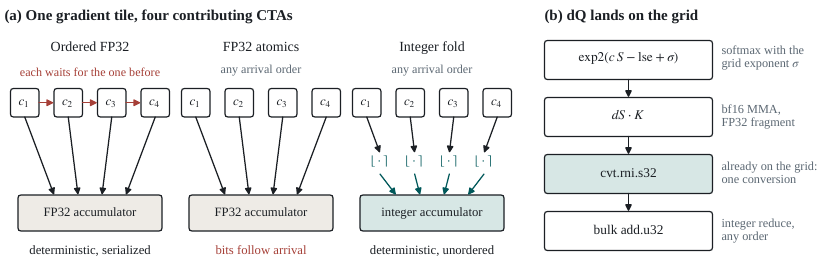}
\caption{\captiona\ Four CTAs' partials into one gradient tile: ordered FP32
adds wait for each other, atomics follow arrival order, and the integer fold
rounds each partial onto one grid ($\lfloor\cdot\rceil$) and adds in any
order. \captionb\ The grid exponent $\sigma$ rides the softmax exponential, so
$dQ$ partials leave the tensor core on the grid.}
\label{fig:backward-reduction}
\end{figure}

\paragraph{dQ.}
A preprocess computes $D_i=dO_i^\top o_i$ and the maxima that bound $dQ$
(Equation~\ref{eq:dq-bound}) per request and KV head. Each warpgroup then
recomputes $P\,2^{\sigma}=\operatorname{exp2}(c\,S-\ell+\sigma)$, placing the
grid exponent $\sigma$ inside the exponential that softmax already requires.
The product $dS\,K$ therefore leaves the tensor core on the grid. One
\texttt{cvt.rni.s32} rounds it, and a store warp adds the tile into a 32-bit
global accumulator with a bulk asynchronous reduction. The grid leaves 30
bits for the bound, so the sum cannot
overflow (Proposition~\ref{prop:grid} with $b=30$). A postprocess scales the
integers back (Algorithm~\ref{alg:backward}).

\paragraph{dK and dV.}
Keys are shared by the $G$ query heads of a GQA group, so $dK$ and $dV$ also
receive partials from several heads. A planner chooses ownership from the
request shape. When there is enough parallel work, one CTA walks a key block
across all $G$ heads and keeps $dK,dV$ in registers. Wider groups split into
subgroups; the last CTA to arrive adds their FP32 partials in subgroup order.
When each head has its own CTA, the CTAs instead round their partials onto
integer grids, as for $dQ$.

\paragraph{Scheduling and guarantees.}
Dense batches run on a persistent kernel that claims tiles from a work list
sorted by cost within sections of (request, head) pairs sized to a fixed
working-set budget, so heavy causal tiles start first while neighbouring CTAs
share their query-side reads in L2. Because no reduction waits on an order,
CTAs claim whatever tile is next, and reversing the list leaves every bit
unchanged. A request's gradients are the same bits alone or packed with
others, because the variable-length plan reads only head counts.

\section{Evaluation}
\label{sec:evaluation}

\paragraph{Setup.}
One H100 80GB HBM3, CUDA 13.0, PyTorch 2.13~\citep{paszke2019pytorch}.
Baselines are FlashAttention-3~\citep{shah2024flashattention3} and
-4~\citep{zadouri2026flashattention4}, FlashInfer~\citep{ye2025flashinfer}
with its tensor-core and TensorRT-LLM XQA decode~\citep{nvidia2026tensorrtllm},
cuDNN~\citep{nvidia2026cudnn}, and DASH~\citep{qiang2026dash}; for shared
prefixes and drafts we add vLLM's cascade path, PAT~\citep{yi2026pat},
FastTree~\citep{pan2025fasttree}, and SGLang's tree
verification~\citep{zheng2024sglang} (versions in
Appendix~\ref{app:protocol}). Each library runs every configuration it offers
at each shape and is represented by its fastest; every kernel is scored
against one FP32 reference. We replay CUDA graphs with L2 evicted before each
sample and the kernel order rotated each round, and report medians of paired
per-round ratios.

\subsection{Decode on model generations}
\label{sec:eval-decode}

We capture post-RoPE~\citep{su2024roformer} queries, keys, and values from prefills of
Qwen3-30B-A3B (layer 24, $D=128$, $G=8$), gpt-oss-20b (layers 9 and 21,
$D=64$, $G=8$), and GLM-4-9B (layers 8 and 28, $D=128$, $G=16$), decode
eight steps with the model's own next tokens, and time every kernel on the
final state (Table~\ref{tab:decode}; Figure~\ref{fig:pareto} plots latency
against error along the depth dial). Dense FoldAttention is faster than every
baseline in all seven cases: 1.24--1.30$\times$ at $G=8$ and
1.14--1.17$\times$ at $G=16$. Its error is within 1.2\% of the lowest BF16
error and below it in six cases. At depth 16, decode is 1.36--2.30$\times$
faster, with an error within 1.5\% of the lowest BF16 error in six cases and
17\% above it on gpt-oss-20b layer 21. Depth 14 is 1.54--2.65$\times$
faster at 1.07--1.46$\times$ that error, still below XQA's in every case. The front kernel adds 4--8\,$\mu$s per step.

\begin{table}[t]
\centering
\caption{Decode on eight-step generations: latency in $\mu$s / FP32-relative
$\ell_2$ error in $10^{-3}$ (speedup over the fastest BF16 baseline).
$\dagger$: error above the lower end of the BF16 baselines' range.}
\label{tab:decode}
\scriptsize
\setlength{\tabcolsep}{3pt}
\begin{tabular}{llcccc}
\toprule
Model, layer, batch $\times$ context & Fastest BF16 & BF16 err. & Fold dense & Fold $T{=}16$ & Fold $T{=}14$ \\
\midrule
Qwen3-30B L24, 8$\times$16K & FA-3 99 & 1.66--3.00 & 77 / 1.63 (1.30$\times$) & 54 / 1.68$^\dagger$ (1.84$\times$) & 45 / 2.33$^\dagger$ (2.22$\times$) \\
Qwen3-30B L24, 32$\times$8--16K & FlashInfer 264 & 1.71--2.33 & 206 / 1.67 (1.28$\times$) & 126 / 1.70 (2.10$\times$) & 103 / 2.19$^\dagger$ (2.54$\times$) \\
Qwen3-30B L24, 16$\times$16--32K & FlashInfer 262 & 1.71--3.13 & 207 / 1.65 (1.27$\times$) & 114 / 1.73$^\dagger$ (2.30$\times$) & 99 / 2.39$^\dagger$ (2.65$\times$) \\
gpt-oss-20b L9, 32$\times$8--16K & cuDNN 269 & 1.68--2.83 & 212 / 1.67 (1.27$\times$) & 163 / 1.70$^\dagger$ (1.65$\times$) & 134 / 2.06$^\dagger$ (2.00$\times$) \\
gpt-oss-20b L21, 32$\times$8--16K & cuDNN 267 & 1.68--2.72 & 216 / 1.70$^\dagger$ (1.24$\times$) & 171 / 1.97$^\dagger$ (1.57$\times$) & 142 / 2.46$^\dagger$ (1.89$\times$) \\
GLM-4-9B L8, 32$\times$8--16K & FlashInfer 139 & 1.70--3.05 & 119 / 1.66 (1.17$\times$) & 91 / 1.67 (1.53$\times$) & 83 / 1.82$^\dagger$ (1.69$\times$) \\
GLM-4-9B L28, 32$\times$8--16K & FlashInfer 140 & 1.69--2.95 & 123 / 1.66 (1.14$\times$) & 103 / 1.69$^\dagger$ (1.36$\times$) & 91 / 2.09$^\dagger$ (1.54$\times$) \\
\bottomrule
\end{tabular}
\end{table}

\subsection{Ablation}
\label{sec:eval-ablation}

Table~\ref{tab:ablation} adds one mechanism at a time. Reading both key
planes costs the same bytes as a BF16 cache, and in this configuration
FoldAttention, despite its extra refinement product, runs at
0.98--1.01$\times$ the fastest vendor-tuned BF16 kernel. Its speedup comes
from the reads that final weights eliminate. Gating the second plane refines
2--29\% of keys at $D=128$, reduces traffic by 18--24\%, and yields
1.18--1.27$\times$ speedup.

\paragraph{Why the reference is declared.}
Online softmax could gate the same reads against its running maximum, which
bounds each final weight from above. Emulated on the seven generations with
the kernel's logits, gates, and splits (Appendix~\ref{app:methods}), that
policy reads 481--510 bytes per key at $D=128$, against 390--426 under the
declared reference, because each split restarts its maximum. Even with one
split per request, it reads 3--16\% more in dense decode and
1.2--1.7$\times$ as many bytes at depth 14. At the BF16 kernels' error, no
other method we emulate (Quest~\citep{tang2024quest}, Faster Flash Decoding,
KIVI, INT8 and FP8 caches) reads fewer bytes than a BF16 kernel; depth 14
reads 117--279.

\begin{table}[t]
\centering
\caption{Decode kernel speed over the fastest BF16 baseline's decode and bytes
read per key, adding one mechanism at a time (ragged batches of 256 KV heads,
16K context). Errors
are FP32-relative $\ell_2$ in $10^{-3}$; the most accurate BF16 baseline's is 1.67--1.69
in these cells, and XQA's 2.31--2.83.}
\label{tab:ablation}
\scriptsize
\setlength{\tabcolsep}{4pt}
\begin{tabular}{lccccc}
\toprule
& Qwen & Qwen & GLM & gpt-oss & \\
& $D{=}128$, $G{=}8$ & $D{=}128$, $G{=}4$ & $D{=}128$, $G{=}16$ & $D{=}64$, $G{=}8$ & Error \\
\midrule
Both planes & 0.98$\times$ (514\,B) & 0.99$\times$ (514\,B) & 0.98$\times$ (514\,B) & 1.00$\times$ (258\,B) & 1.64--1.66 \\
Gated plane B (dense) & 1.26$\times$ (390\,B) & 1.27$\times$ (389\,B) & 1.18$\times$ (423\,B) & 1.24$\times$ (204\,B) & 1.66--1.67 \\
+ 8-bit values & 1.76$\times$ (262\,B) & 1.79$\times$ (261\,B) & 1.41$\times$ (275\,B) & 1.56$\times$ (134\,B) & 1.68--1.70 \\
+ depth $T{=}16$, BF16 values & 2.15$\times$ (217\,B) & 2.37$\times$ (196\,B) & 1.53$\times$ (308\,B) & 1.65$\times$ (141\,B) & 1.70--1.72 \\
\bottomrule
\end{tabular}
\end{table}

\subsection{Context and group-size sweep}
Figure~\ref{fig:sweep} sweeps context from 1K to 32K keys for MHA and GQA
groups of eight at $D=128$ and $D=64$, using ragged batches of 256 KV heads.
At $D=128$, dense FoldAttention is 1.19--1.25$\times$ faster at 1K and
1.27--1.29$\times$ at 32K, where it reads at 95--97\% of the measured
bandwidth ceiling. At $D=64$, it is 1.04--1.10$\times$ faster at 1K and
1.30--1.35$\times$ at 32K. MHA gains most, because no verdict is ORed across
a group. At 32K, depth 16 reaches 3.02$\times$ at $D=128$ and 2.83$\times$
at $D=64$. Across all 84 cells of both head
dimensions, groups of one to sixteen, and ragged and uniform batches
(Appendix~\ref{app:decode}), dense decode reaches 1.35$\times$ and is slower
only once (0.95$\times$ for uniform $D=64$, $G=8$ at 1K). Depth 16 reaches
3.09$\times$ while remaining within the BF16 error range in 82 cells. Depth
14 reaches 3.50$\times$ and remains within that range in 67. Dense decode
with 8-bit values reaches 1.89$\times$.

\begin{figure}[t]
\centering
\includegraphics[width=\linewidth]{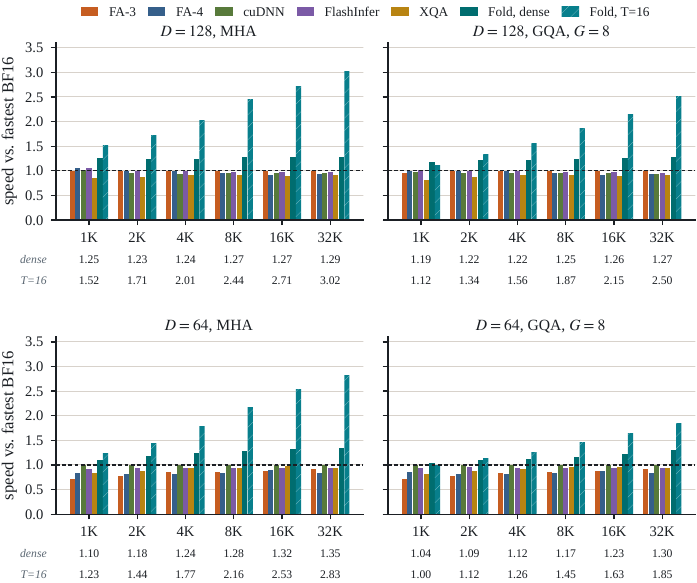}
\caption{Decode speed relative to the fastest BF16 baseline in each cell
(dashed line at 1), on ragged batches of 256 KV heads at $D=128$ and $D=64$,
for MHA and GQA with $G=8$. The table under each panel gives FoldAttention's
speedups, dense and at depth 16.}
\label{fig:sweep}
\end{figure}

\subsection{Serving and quality}
\label{sec:eval-e2e}

We decode Qwen3-8B~\citep{yang2025qwen3} ($G=4$, $D=128$) with the whole
model step captured in one CUDA graph, changing only the attention call and
its cache, at 32 requests of 8K, 16 of 16K, and 8 of 32K. Dense FoldAttention
makes the whole step 1.13--1.14$\times$ faster. Depth 16 gives
1.35--1.46$\times$, and depth 14 gives 1.44--1.54$\times$. With BF16 values,
a finite depth needs the block model, fitted on the host in 6.6--11.3\,ms per
layer per prefill, which the depth-16 step recovers within 38--65 generated
tokens; with 8-bit values it costs nothing (Appendix~\ref{app:decode}).

Each arm decodes into its own cache and is teacher-forced after 8K--32K
prompts (Table~\ref{tab:quality}). FoldAttention's KL divergence from BF16
FlashAttention-3 is 8.0--8.5$\times10^{-4}$ nats per token, compared with
7.9--8.9$\times10^{-4}$ for FlashInfer and cuDNN, and its likelihood differs
by at most $5.5\times10^{-4}$ nats. The FP8 cache has 9--16$\times$ that
divergence and a likelihood up to $7.5\times10^{-3}$ nats worse. At 64K and
128K, FoldAttention remains at FlashInfer's divergence level
(8.9--10.5 versus 9.1--9.8$\times10^{-4}$), while FP8 reaches
$175\times10^{-4}$. Its mean RULER accuracy over four tasks is 67.6--68.0\%,
against 68.0\% for FlashAttention-3, 66.8\% for FlashInfer, and 66.0\% for FP8. On the 16 English tasks of LongBench~\citep{bai2024longbench}, its mean
score is 49.2--49.3 against 49.3 for FlashAttention-3.
Greedy generations of 1024 tokens first leave FlashAttention-3's after a
median of 49--90 tokens with FoldAttention, 55--61 with the other BF16
kernels, and 26 with FP8.

\begin{table}[t]
\centering
\caption{Qwen3-8B decoding for real in each arm. KL is FlashAttention-3's
next-token distribution against the arm's, in $10^{-4}$ nats per token, and
$\Delta$NLL the arm's teacher-forced negative log-likelihood minus
FlashAttention-3's, in $10^{-4}$ nats, on 32 WikiText-103 documents of 1024
tokens. RULER is the mean exact match over four tasks. Contexts past 32K use
YaRN~\citep{peng2024yarn}; their cuDNN and 8-bit $V$ cells come from a
second run, in which FlashAttention-3 scores 70.3 on RULER. ``Diverge'' is
the median position at which greedy generation first leaves
FlashAttention-3's.}
\label{tab:quality}
\scriptsize
\setlength{\tabcolsep}{3pt}
\begin{tabular}{lccccccccc}
\toprule
& \multicolumn{3}{c}{KL, 8K / 16K / 32K} & $\Delta$NLL & RULER & KL & RULER & LongBench & Diverge \\
Arm & & & & 32K & 8--32K & 64K / 128K & 64--128K & 16 tasks & \\
\midrule
FA-3 & 0 & 0 & 0 & 0 & 87.1 & 0 / 0 & 68.0 & 49.3 & -- \\
FlashInfer & 7.9 & 8.0 & 8.6 & $-2.9$ & 86.9 & 9.1 / 9.8 & 66.8 & 49.3 & 55 \\
cuDNN & 8.0 & 8.1 & 8.9 & $-1.1$ & 87.1 & 8.8 / 10.2 & 69.9 & -- & 61 \\
FlashInfer FP8 & 73.9 & 109.0 & 131.8 & $+75.1$ & 87.3 & 124.8 / 174.6 & 66.0 & 49.2 & 26 \\
Fold dense & 8.1 & 8.1 & 8.5 & $+0.2$ & 87.6 & 8.9 / 10.5 & 67.6 & 49.3 & 90 \\
Fold dense, 8-bit $V$ & 8.0 & 8.0 & 8.1 & $+0.1$ & 87.1 & 9.0 / 9.6 & 69.5 & -- & 78 \\
Fold $T{=}16$ & 8.1 & 8.1 & 8.2 & $+0.0$ & 86.6 & 8.9 / 10.3 & 68.0 & 49.3 & 49 \\
Fold $T{=}14$ & 8.1 & 8.1 & 8.5 & $+4.3$ & 86.9 & 9.1 / 9.8 & 67.6 & 49.2 & 67 \\
\bottomrule
\end{tabular}
\end{table}

\subsection{Cascades and drafts}
\label{sec:eval-compose}

On 16 shared-prefix cascades of Qwen3-30B-A3B and gpt-oss-20b, dense
FoldAttention is faster than every kernel we ran in every cell,
1.04--1.42$\times$ over the fastest, and 1.34$\times$, 1.32$\times$, and
1.61$\times$ faster by geometric mean than vLLM's cascade path, PAT, and
FastTree; on prefix trees it matches PAT (1.07$\times$). Verifying chains of
two or four drafts as extra query columns, it is 0.99--1.24$\times$ as fast
as the fastest library; longer chains and trees favor FA-3 and SGLang
(Appendix~\ref{app:compose}).

\subsection{Backward}
\label{sec:eval-backward}

On eight GQA and MHA model shapes at 16K tokens or more
(Table~\ref{tab:backward}), FoldAttention is 1.29$\times$ and 1.23$\times$
faster than deterministic FlashAttention-3 and -4, 1.32$\times$ faster than
DASH, and 1.05$\times$ faster than the fastest nondeterministic kernel;
packed variable-length batches behave the same. On the grid FlashAttention-3
and DASH report, it is 1.12$\times$ faster than DASH with MHA, and with GQA
groups of eight up to 1.84$\times$ faster than deterministic FlashAttention
and 1.87$\times$ faster than DASH (Appendix~\ref{app:backward}). In the same
kernel, FP32 atomics for $dQ$ are 1.0\% faster over 32 shapes, while replacing
the persistent work list with one tile per CTA is 3.9\% slower. The order-free
sum is what lets a deterministic kernel claim tiles in any order; the
deterministic modes of FlashAttention-3 and -4 take 24\% and 19\% longer
than their nondeterministic modes.

\begin{table}[t]
\centering
\caption{Backward speedup of FoldAttention over each baseline (geometric
mean). ``Fastest'' takes the best kernel of its kind at each shape. Training
time combines FA-4's forward with each backward.}
\label{tab:backward}
\scriptsize
\setlength{\tabcolsep}{3.5pt}
\begin{tabular}{lcccccccc}
\toprule
Set & FA-3 & FA-3 det. & FA-4 & FA-4 det. & cuDNN & DASH & Fastest det. & Fastest nondet. \\
\midrule
Model shapes, backward (8) & 1.07 & 1.29 & 1.06 & 1.23 & 1.23 & 1.32 & 1.23 & 1.05 \\
Packed varlen, backward (5) & 1.08 & 1.30 & 1.07 & 1.29 & -- & -- & 1.28 & 1.06 \\
MHA grid, backward (24) & 0.99 & 1.18 & 0.98 & 1.11 & 1.17 & 1.12 & 1.07 & 0.97 \\
GQA grid ($G{=}8$), backward (24) & 1.06 & 1.32 & 1.04 & 1.25 & 1.23 & 1.28 & 1.23 & 1.04 \\
Model shapes, train (8) & 0.99 & 1.15 & 1.01 & 1.14 & 1.26 & 1.17 & 1.12 & 0.99 \\
\bottomrule
\end{tabular}
\end{table}

\paragraph{Determinism and training.}
FoldAttention's gradients are bit-identical when a request is rerun, batched,
or packed, as are those of the deterministic modes of FlashAttention-3 and
-4 and of DASH, while the nondeterministic kernels repeat their bits on at
most 8\% of shapes (Appendix~\ref{app:determinism}). On operands captured from
training, FoldAttention's $dQ$, $dK$, and $dV$ match FlashAttention-3's
$\ell_2$ error to three digits; their per-element relative errors are also
comparable through the 99th percentile (Table~\ref{tab:grad-elements}). With
FlashAttention-4's forward,
a forward and backward step is 1.12$\times$ faster than the fastest
deterministic alternative, and a whole-model training step of a 1B
Llama~\citep{grattafiori2024llama3} is within 0.5\% of the nondeterministic
kernels. Three FoldAttention runs trained from scratch for 2000 steps of 16K
tokens produce bit-identical weights. Their validation loss is 3.946, within
the 3.939--3.951 range of the other kernels
(Figure~\ref{fig:train-curve}).

\section{Related work}
\label{sec:related}

\paragraph{Attention kernels and state merges.}
Hydragen, vLLM's cascade path, PAT, and FastTree read a shared prefix once and
merge its online-softmax state into each
request's~\citep{juravsky2024hydragen,kwon2023pagedattention,yi2026pat,pan2025fasttree},
as SGLang does for draft trees~\citep{zheng2024sglang}; FoldAttention adds
the levels' pairs instead.

\paragraph{Fixed softmax references.}
FlashDecoding++ applies one profiled constant, with a synchronized fallback
when scores leave its FP16 range~\citep{hong2024flashdecodingpp};
FlashAttention-4 skips small rescales~\citep{zadouri2026flashattention4}; VFA
freezes a maximum from block summaries~\citep{sun2026vfa}; and Faster Flash
Decoding and BLASST skip blocks below a pseudo- or running maximum,
discarding their mass~\citep{liu2026ffd,yuan2025blasst}. FoldAttention needs
only a reference in the exponent range, and uses final weights for per-key
reads, additive merges, and a dense denominator at every depth.

\paragraph{Quantized caches and sparse decode.}
KIVI, KVQuant, and FP8 caches compress the KV
cache~\citep{liu2024kivi,hooper2024kvquant,micikevicius2022fp8}, and Quest,
SparQ, and CoSA select pages or skip value reads by approximate
scores~\citep{tang2024quest,ribar2023sparq,xue2026cosa}. These methods
primarily evaluate task accuracy rather than agreement with a BF16 kernel.
Among the methods we emulate, none reads fewer bytes than a BF16 kernel at
its output error (Section~\ref{sec:eval-ablation}).

\paragraph{Reproducible reduction.}
Pre-rounding summands to a common grid makes floating-point sums
reproducible~\citep{demmel2013fast,ahrens2020reprotoms}. DASH reschedules
FlashAttention's ordered backward as a DAG~\citep{qiang2026dash}, and
batch-invariant kernels fix reduction topologies and split sizes for serving
and for reinforcement learning, whose rollouts must match the
trainer~\citep{zhang2025tbik,he2025nondeterminism,sglang2025deterministic,zhong2026vexact}.
FoldAttention pre-rounds inside the attention backward with bounds it
computes cheaply, so determinism needs no ordering.

\section{Limitations and conclusion}
\label{sec:limitations}

The kernels target Hopper, head dimensions 64 and 128 for decode (96 also for
backward), and MHA and GQA caches; latent caches such as MLA's are not yet
supported, and porting the WGMMA layouts to Blackwell is future work. Online
softmax discovers a row's reference; FoldAttention declares it. Final weights
let decode read only the bytes they need, from BF16-kernel accuracy to three
times faster on one cache, and a declared grid lets the backward sum integers
in any order, deterministic and faster than the nondeterministic kernels.

\bibliography{references}

@inproceedings{dao2022flashattention,
  title={{FlashAttention}: Fast and Memory-Efficient Exact Attention with {IO}-Awareness},
  author={Dao, Tri and Fu, Daniel Y. and Ermon, Stefano and Rudra, Atri and R{\'e}, Christopher},
  booktitle={Advances in Neural Information Processing Systems},
  year={2022}
}

@article{dao2023flashattention2,
  title={{FlashAttention-2}: Faster Attention with Better Parallelism and Work Partitioning},
  author={Dao, Tri},
  journal={arXiv preprint arXiv:2307.08691},
  year={2023}
}

@article{shah2024flashattention3,
  title={{FlashAttention-3}: Fast and Accurate Attention with Asynchrony and Low-Precision},
  author={Shah, Jay and Bikshandi, Ganesh and Zhang, Ying and Thakkar, Vijay and Ramani, Pradeep and Dao, Tri},
  journal={arXiv preprint arXiv:2407.08608},
  year={2024}
}

@article{zadouri2026flashattention4,
  title={{FlashAttention-4}: Algorithm and Kernel Pipelining Co-Design for Asymmetric Hardware Scaling},
  author={Zadouri, Ted and Hoehnerbach, Markus and Shah, Jay and Liu, Timmy and Thakkar, Vijay and Dao, Tri},
  journal={arXiv preprint arXiv:2603.05451},
  year={2026}
}

@article{sanovar2024leanattention,
  title={Lean Attention: Hardware-Aware Scalable Attention Mechanism for the Decode-Phase of Transformers},
  author={Sanovar, Rya and Bharadwaj, Srikant and St. Amant, Renee and R{\"u}hle, Victor and Rajmohan, Saravan},
  journal={arXiv preprint arXiv:2405.10480},
  year={2024}
}

@article{ye2025flashinfer,
  title={{FlashInfer}: Efficient and Customizable Attention Engine for {LLM} Inference Serving},
  author={Ye, Zihao and Chen, Lequn and Lai, Ruihang and Lin, Wuwei and Zhang, Yineng and Wang, Stephanie and Chen, Tianqi and Kasikci, Baris and Grover, Vinod and Krishnamurthy, Arvind and Ceze, Luis},
  journal={arXiv preprint arXiv:2501.01005},
  year={2025}
}

@article{qiang2026dash,
  title={{DASH}: Deterministic Attention Scheduling for High-throughput Reproducible {LLM} Training},
  author={Qiang, Xinwei and Chen, Hongmin and Sun, Shixuan and Leng, Jingwen and Liu, Xin and Guo, Minyi},
  journal={arXiv preprint arXiv:2601.21824},
  year={2026}
}

@article{zhang2025tbik,
  title={Deterministic Inference across Tensor Parallel Sizes That Eliminates Training-Inference Mismatch},
  author={Zhang, Ziyang and Ding, Xinheng and Yuan, Jiayi and Liu, Rixin and Mao, Huizi and Xing, Jiarong and Liu, Zirui},
  journal={arXiv preprint arXiv:2511.17826},
  year={2025}
}

@article{zhong2026vexact,
  title={Diagnosing Training Inference Mismatch in {LLM} Reinforcement Learning},
  author={Zhong, Tianle and Ling, Neiwen and Pi, Yifan and Wei, Zijun and Yu, Tianshu and Fox, Geoffrey and Wu, Peng and Yu, Xiao},
  journal={arXiv preprint arXiv:2605.14220},
  year={2026}
}

@misc{nvidia2026cudnn,
  title={{NVIDIA cuDNN}: Scaled Dot Product Attention},
  author={{NVIDIA Corporation}},
  year={2026},
  howpublished={\url{https://docs.nvidia.com/deeplearning/cudnn/latest/operations/Attention.html}},
  note={Accessed September 19, 2026}
}

@inproceedings{hong2024flashdecodingpp,
  title={{FlashDecoding++}: Faster Large Language Model Inference on {GPUs}},
  author={Hong, Ke and Dai, Guohao and Xu, Jiaming and Mao, Qiuli and Li, Xiuhong and Liu, Jun and Chen, Kangdi and Dong, Yuhan and Wang, Yu},
  booktitle={Proceedings of Machine Learning and Systems},
  year={2024}
}

@inproceedings{kwon2023pagedattention,
  title={Efficient Memory Management for Large Language Model Serving with {PagedAttention}},
  author={Kwon, Woosuk and Li, Zhuohan and Zhuang, Siyuan and Sheng, Ying and Zheng, Lianmin and Yu, Cody Hao and Gonzalez, Joseph E. and Zhang, Hao and Stoica, Ion},
  booktitle={Proceedings of the 29th Symposium on Operating Systems Principles},
  year={2023}
}

@article{juravsky2024hydragen,
  title={{Hydragen}: High-Throughput {LLM} Inference with Shared Prefixes},
  author={Juravsky, Jordan and Brown, Bradley and Ehrlich, Ryan and Fu, Daniel Y. and R{\'e}, Christopher and Mirhoseini, Azalia},
  journal={arXiv preprint arXiv:2402.05099},
  year={2024}
}

@article{ribar2023sparq,
  title={{SparQ Attention}: Bandwidth-Efficient {LLM} Inference},
  author={Ribar, Luka and Chelombiev, Ivan and Hudlass-Galley, Luke and Blake, Charlie and Luschi, Carlo and Orr, Douglas},
  journal={arXiv preprint arXiv:2312.04985},
  year={2023}
}

@inproceedings{liu2024kivi,
  title={{KIVI}: A Tuning-Free Asymmetric 2bit Quantization for {KV} Cache},
  author={Liu, Zirui and Yuan, Jiayi and Jin, Hongye and Zhong, Shaochen and Xu, Zhaozhuo and Braverman, Vladimir and Chen, Beidi and Hu, Xia},
  booktitle={International Conference on Machine Learning},
  year={2024}
}

@inproceedings{hooper2024kvquant,
  title={{KVQuant}: Towards 10 Million Context Length {LLM} Inference with {KV} Cache Quantization},
  author={Hooper, Coleman and Kim, Sehoon and Mohammadzadeh, Hiva and Mahoney, Michael W. and Shao, Yakun Sophia and Keutzer, Kurt and Gholami, Amir},
  booktitle={Advances in Neural Information Processing Systems},
  year={2024}
}

@article{liu2026ffd,
  title={Faster Than Flash: Exploiting Attention Sparsity for Efficient Long-Context Decoding},
  author={Liu, Zhigeng and Ning, Zhiyuan and Li, Ruixiao and Liu, Xiaoran and Song, Yuerong and Zhang, Min and He, Ziwei and Qiu, Xipeng},
  journal={arXiv preprint arXiv:2609.00097},
  year={2026}
}

@article{xue2026cosa,
  title={{CoSA}: Accelerating Long-Context Inference via Proxy-Kernel Co-Designed Sparse Attention},
  author={Xue, Yufei and Niu, Lin and Liu, Hong and Liu, Siran and Shao, Hanyong and Liu, Wei and Yu, Guanghua and Zhu, Jianchen and Zhang, Jun},
  journal={arXiv preprint arXiv:2607.25291},
  year={2026}
}

@article{sun2026vfa,
  title={VFA: Relieving Vector Operations in Flash Attention with Global Maximum Pre-computation},
  author={Sun, Yupeng and Li, Yanzhao and Zou, Zhiqiang and Du, Bai and Zhang, Zhiyuan and Dong, Hui and Fan, Gaoyige and Wang, Hui},
  journal={arXiv preprint arXiv:2604.12798},
  year={2026}
}

@article{milakov2018online,
  title={Online normalizer calculation for softmax},
  author={Milakov, Maxim and Gimelshein, Natalia},
  journal={arXiv preprint arXiv:1805.02867},
  year={2018}
}

@inproceedings{demmel2013fast,
  title={Fast Reproducible Floating-Point Summation},
  author={Demmel, James and Nguyen, Hong Diep},
  booktitle={IEEE Symposium on Computer Arithmetic (ARITH)},
  year={2013}
}

@article{ahrens2020reprotoms,
  title={Algorithms for Efficient Reproducible Floating Point Summation},
  author={Ahrens, Peter and Demmel, James and Nguyen, Hong Diep},
  journal={ACM Transactions on Mathematical Software},
  volume={46},
  number={3},
  year={2020}
}

@inproceedings{ashkboos2024quarot,
  title={{QuaRot}: Outlier-Free 4-Bit Inference in Rotated {LLMs}},
  author={Ashkboos, Saleh and Mohtashami, Amirkeivan and Croci, Maximilian L. and Li, Bo and Cameron, Pashmina and Jaggi, Martin and Alistarh, Dan and Hoefler, Torsten and Hensman, James},
  booktitle={Advances in Neural Information Processing Systems},
  year={2024}
}

@misc{he2025nondeterminism,
  title={Defeating Nondeterminism in {LLM} Inference},
  author={He, Horace and {Thinking Machines Lab}},
  howpublished={Thinking Machines Lab: Connectionism},
  year={2025},
  note={\url{https://thinkingmachines.ai/blog/defeating-nondeterminism-in-llm-inference/}}
}

@inproceedings{vaswani2017attention,
  title={Attention Is All You Need},
  author={Vaswani, Ashish and Shazeer, Noam and Parmar, Niki and Uszkoreit, Jakob and Jones, Llion and Gomez, Aidan N. and Kaiser, {\L}ukasz and Polosukhin, Illia},
  booktitle={Advances in Neural Information Processing Systems},
  year={2017}
}

@inproceedings{ainslie2023gqa,
  title={{GQA}: Training Generalized Multi-Query Transformer Models from Multi-Head Checkpoints},
  author={Ainslie, Joshua and Lee-Thorp, James and de Jong, Michiel and Zemlyanskiy, Yury and Lebr{\'o}n, Federico and Sanghai, Sumit},
  booktitle={Proceedings of the 2023 Conference on Empirical Methods in Natural Language Processing},
  year={2023}
}

@article{su2024roformer,
  title={{RoFormer}: Enhanced Transformer with Rotary Position Embedding},
  author={Su, Jianlin and Ahmed, Murtadha and Lu, Yu and Pan, Shengfeng and Bo, Wen and Liu, Yunfeng},
  journal={Neurocomputing},
  volume={568},
  year={2024}
}

@inproceedings{xiao2024streamingllm,
  title={Efficient Streaming Language Models with Attention Sinks},
  author={Xiao, Guangxuan and Tian, Yuandong and Chen, Beidi and Han, Song and Lewis, Mike},
  booktitle={International Conference on Learning Representations},
  year={2024}
}

@article{micikevicius2022fp8,
  title={{FP8} Formats for Deep Learning},
  author={Micikevicius, Paulius and Stosic, Dusan and Burgess, Neil and Cornea, Marius and Dubey, Pradeep and Grisenthwaite, Richard and Ha, Sangwon and Heinecke, Alexander and Judd, Patrick and Kamalu, John and Mellempudi, Naveen and Oberman, Stuart and Shoeybi, Mohammad and Siu, Michael and Wu, Hao},
  journal={arXiv preprint arXiv:2209.05433},
  year={2022}
}

@inproceedings{leviathan2023fast,
  title={Fast Inference from Transformers via Speculative Decoding},
  author={Leviathan, Yaniv and Kalman, Matan and Matias, Yossi},
  booktitle={International Conference on Machine Learning},
  year={2023}
}

@inproceedings{miao2024specinfer,
  title={{SpecInfer}: Accelerating Large Language Model Serving with Tree-based Speculative Inference and Verification},
  author={Miao, Xupeng and Oliaro, Gabriele and Zhang, Zhihao and Cheng, Xinhao and Wang, Zeyu and Zhang, Zhengxin and Wong, Rae Ying Yee and Zhu, Alan and Yang, Lijie and Shi, Xiaoxiang and Shi, Chunan and Chen, Zhuoming and Arfeen, Daiyaan and Abhyankar, Reyna and Jia, Zhihao},
  booktitle={Proceedings of the 29th ACM International Conference on Architectural Support for Programming Languages and Operating Systems},
  year={2024}
}

@misc{dao2023flashdecoding,
  title={Flash-Decoding for Long-Context Inference},
  author={Dao, Tri and Haziza, Daniel and Massa, Francisco and Sizov, Grigory},
  howpublished={\url{https://crfm.stanford.edu/2023/10/12/flashdecoding.html}},
  year={2023}
}

@inproceedings{paszke2019pytorch,
  title={{PyTorch}: An Imperative Style, High-Performance Deep Learning Library},
  author={Paszke, Adam and Gross, Sam and Massa, Francisco and Lerer, Adam and Bradbury, James and Chanan, Gregory and Killeen, Trevor and Lin, Zeming and Gimelshein, Natalia and Antiga, Luca and Desmaison, Alban and K{\"o}pf, Andreas and Yang, Edward and DeVito, Zachary and Raison, Martin and Tejani, Alykhan and Chilamkurthy, Sasank and Steiner, Benoit and Fang, Lu and Bai, Junjie and Chintala, Soumith},
  booktitle={Advances in Neural Information Processing Systems},
  year={2019}
}

@misc{nvidia2026tensorrtllm,
  title={{TensorRT-LLM}},
  author={{NVIDIA Corporation}},
  howpublished={\url{https://github.com/NVIDIA/TensorRT-LLM}},
  year={2026}
}

@article{yang2025qwen3,
  title={{Qwen3} Technical Report},
  author={Yang, An and others},
  journal={arXiv preprint arXiv:2505.09388},
  year={2025}
}

@article{openai2025gptoss,
  title={gpt-oss-120b \& gpt-oss-20b Model Card},
  author={{OpenAI}},
  journal={arXiv preprint arXiv:2508.10925},
  year={2025}
}

@article{glm2024chatglm,
  title={{ChatGLM}: A Family of Large Language Models from {GLM-130B} to {GLM-4} All Tools},
  author={{Team GLM}},
  journal={arXiv preprint arXiv:2406.12793},
  year={2024}
}

@article{grattafiori2024llama3,
  title={The {Llama} 3 Herd of Models},
  author={Grattafiori, Aaron and others},
  journal={arXiv preprint arXiv:2407.21783},
  year={2024}
}

@article{yuan2025blasst,
  title={{BLASST}: Dynamic {BLocked} Attention Sparsity via Softmax Thresholding},
  author={Yuan, Jiayi and Shinn, Cameron and Xu, Kai and Cui, Jingze and Klimiashvili, George and Xiao, Guangxuan and Zheng, Perkz and Li, Bo and Zhou, Yuxin and Ye, Zhouhai and You, Weijie and Zheng, Tian and Brown, Dominic and Wang, Pengbo and Hoehnerbach, Markus and Cai, Richard and Demouth, Julien and Owens, John D. and Hu, Xia and Han, Song and Liu, Timmy and Mao, Huizi},
  journal={arXiv preprint arXiv:2512.12087},
  year={2025}
}

@book{higham2002accuracy,
  title={Accuracy and Stability of Numerical Algorithms},
  author={Higham, Nicholas J.},
  edition={2},
  publisher={SIAM},
  year={2002}
}

@inproceedings{yi2026pat,
  title={{PAT}: Accelerating {LLM} Decoding via Prefix-Aware Attention with Resource Efficient Multi-Tile Kernel},
  author={Yi, Jinjun and Zhao, Zhixin and Hu, Yitao and Yan, Ke and Sun, Weiwei and Wang, Hao and Zhao, Laiping and Zhang, Yuhao and Li, Wenxin and Li, Keqiu},
  booktitle={Proceedings of the 31st ACM International Conference on Architectural Support for Programming Languages and Operating Systems (ASPLOS)},
  year={2026}
}

@inproceedings{pan2025fasttree,
  title={{FastTree}: Optimizing Attention Kernel and Runtime for Tree-Structured {LLM} Inference},
  author={Pan, Zaifeng and Ding, Yitong and Guan, Yue and Wang, Zheng and Yu, Zhongkai and Tang, Xulong and Wang, Yida and Ding, Yufei},
  booktitle={Proceedings of Machine Learning and Systems (MLSys)},
  year={2025}
}

@inproceedings{zheng2024sglang,
  title={{SGLang}: Efficient Execution of Structured Language Model Programs},
  author={Zheng, Lianmin and Yin, Liangsheng and Xie, Zhiqiang and Sun, Chuyue and Huang, Jeff and Yu, Cody Hao and Cao, Shiyi and Kozyrakis, Christos and Stoica, Ion and Gonzalez, Joseph E. and Barrett, Clark and Sheng, Ying},
  booktitle={Advances in Neural Information Processing Systems (NeurIPS)},
  year={2024}
}

@inproceedings{tang2024quest,
  title={{Quest}: Query-Aware Sparsity for Efficient Long-Context {LLM} Inference},
  author={Tang, Jiaming and Zhao, Yilong and Zhu, Kan and Xiao, Guangxuan and Kasikci, Baris and Han, Song},
  booktitle={Proceedings of the 41st International Conference on Machine Learning},
  year={2024}
}

@misc{sglang2025deterministic,
  title={Towards Deterministic Inference in {SGLang} and Reproducible {RL} Training},
  author={{The SGLang Team}},
  howpublished={LMSYS Org Blog},
  year={2025},
  note={\url{https://lmsys.org/blog/2025-09-22-sglang-deterministic/}}
}

@inproceedings{bai2024longbench,
  title={{LongBench}: A Bilingual, Multitask Benchmark for Long Context Understanding},
  author={Bai, Yushi and Lv, Xin and Zhang, Jiajie and Lyu, Hongchang and Tang, Jiankai and Huang, Zhidian and Du, Zhengxiao and Liu, Xiao and Zeng, Aohan and Hou, Lei and Dong, Yuxiao and Tang, Jie and Li, Juanzi},
  booktitle={Proceedings of the 62nd Annual Meeting of the Association for Computational Linguistics},
  year={2024}
}

@inproceedings{peng2024yarn,
  title={{YaRN}: Efficient Context Window Extension of Large Language Models},
  author={Peng, Bowen and Quesnelle, Jeffrey and Fan, Honglu and Shippole, Enrico},
  booktitle={International Conference on Learning Representations},
  year={2024}
}
\bibliographystyle{plainnat}

\clearpage
\appendix
\raggedbottom
\section{Proofs}
\label{app:proofs}

\begin{proof}[Proof of Proposition~\ref{prop:fold}]
(i) Both components of $F_i$ are sums over keys of terms that depend only on
$s_{ij}$, $v_j$, and $Z_i$, so a sum over $I\cup J$ splits into the sums over
$I$ and $J$ when they are disjoint.
(ii) Multiplying numerator and denominator by $2^{Z_i}>0$,
\[
\frac{A_i}{L_i}=\frac{\sum_j2^{s_{ij}-Z_i}v_j}{\sum_j2^{s_{ij}-Z_i}}
=\frac{\sum_j2^{s_{ij}}v_j}{\sum_j2^{s_{ij}}}
=\sum_j\operatorname{softmax}_j\bigl(q_i^\top k_j/\sqrt D\bigr)\,v_j ,
\]
since $2^{s_{ij}}=e^{q_i^\top k_j/\sqrt D}$. The quotient does not depend on
$Z_i$.
(iii) Every weight satisfies $0\le2^{s_{ij}-Z_i}\le2^{m_i-Z_i}$, so the exact
sums obey $L_i\le n\,2^{m_i-Z_i}<2^{127}/\nu$ and
$|A_{id}|\le\nu L_i<2^{127}$. Recursive FP32 summation of nonnegative terms
with unit roundoff $u=2^{-24}$ returns at most $(1+u)^{n}$ times the exact
sum~\citep{higham2002accuracy}, and $(1+u)^{2^{22}}<e^{1/4}<2$; the same
factor bounds $\sum_j2^{s_{ij}-Z_i}|v_{jd}|$, which dominates the computed
$|A_{id}|$. Every intermediate therefore stays below $2^{128}$, the FP32 and
BF16 overflow threshold. With flush to zero, a weight is lost only if it is
below $2^{-126}$, so at most $n$ weights totalling less than $n\,2^{-126}$
are lost, while $L_i\ge2^{m_i-Z_i}$ from the largest weight alone. The
relative change of $L_i$ is at most $n\,2^{-126-(m_i-Z_i)}$.
\end{proof}

On the generations of Section~\ref{sec:evaluation}, $m_i-Z_i$ lies in
$[-8.2,16.7]$ (Figure~\ref{fig:headroom}). With $n\le2^{22}$ and values
below $2^{16}$, the left side of the overflow condition is at most 55, far
below 127, and the flush bound is below $2^{-95}$.

\begin{proof}[Proof of Proposition~\ref{prop:grid}]
Round to nearest moves each term by at most $\tfrac12$, so
$|\hat g_c|\le\alpha|g_c|+\tfrac12$. Because
$\lceil\log_2B\rceil\ge\log_2B$, $\alpha\le2^b/B$, and
\[
\textstyle\sum_c|\hat g_c|\le\alpha\sum_c|g_c|+\tfrac C2\le2^b+\tfrac C2<2^{b+1}.
\]
Any partial sum, over any subset of the $\hat g_c$ in any order, is bounded by
the same quantity and so is representable in $(b+2)$-bit two's complement,
whose range is $[-2^{b+1},2^{b+1})$. Integer addition without overflow is
exact, hence associative and commutative, so every order and grouping gives
$\sum_c\hat g_c$. For the error,
$|\alpha^{-1}\sum_c\hat g_c-\sum_cg_c|\le\alpha^{-1}\sum_c|\hat g_c-\alpha g_c|
\le C/(2\alpha)$, and $\lceil\log_2B\rceil<\log_2B+1$ gives
$\alpha>2^{b-1}/B$, so $C/(2\alpha)<C\,B\,2^{-b}$.
\end{proof}

In the kernel, $g_c$ is the FP32 partial one CTA computes, with the power of
two $\alpha$ already applied inside the exponential, which is exact barring
underflow. The bound covers every partial because
Equation~\ref{eq:dq-bound} bounds the sum of the absolute values of all
terms of the element, of which each partial is a subset, and the $2^{-6}$
widening covers the BF16 rounding of $P$ and $dS$ that the computed partial
carries. The $dQ$ grid uses $b=30$ with 32-bit accumulators; the $dK$ grid,
and the $dV$ grid when $G\cdot S>16384$, use $b=61$ with 64-bit
accumulators.

\clearpage
\section{Algorithms}
\label{app:algorithms}

Algorithm~\ref{alg:decode} gives one decode CTA's loop over the 64-key tiles
of its split for one GQA row group, with BF16 values and two weight terms;
8-bit values add a gated second value plane, a finite depth with the block
tail adds a virtual row per block, and the 128-key tile holds two keys in
each row of the logit (Section~\ref{sec:decode}). Query
row $i$ has INT8 planes $q^a_i,q^b_i$ and scale $\eta_i$, key $j$ has planes
$a_j,b_j$ and scale $e_j$, and scores are in base 2.

\begin{algorithm}[H]
\caption{FoldAttention decode, one CTA}
\label{alg:decode}
\small
\begin{algorithmic}[1]
\Require references $Z_i$, depth $T$ ($\infty$ for dense), refine gate $\tau_K$, key tiles $1,\dots,N$ of this split, every $n$-th tile of the request for $n$ splits
\State $A\gets0\in\mathbb R^{D\times G}$, $L\gets0\in\mathbb R^{G}$ \Comment{FP32}
\State issue a bulk copy of tile 1's plane A and key scales
\For{$t=1,\dots,N$}
  \State wait for tile $t$'s plane A; issue tile $t{+}1$'s
  \State $c^{aa},c^{ab}\gets K^{a}_t\,[Q^{a};Q^{b}]^\top$ \Comment{INT8 WGMMA, keys on $M$}
  \State $\tilde s_{ij}\gets\eta_ie_j\,(c^{aa}_{ji}+c^{ab}_{ji}/256)-Z_i$ \Comment{coarse score against the fixed reference}
  \State $\mathrm{live}_j\gets\bigvee_i[\tilde s_{ij}\ge-T]$;\quad
         $\mathrm{refine}_j\gets\mathrm{live}_j\wedge\bigvee_i[\tilde s_{ij}\ge-\tau_K]$ \Comment{ORed over the group}
  \If{no key is live and there is no block tail}
    \State \textbf{continue}
  \EndIf
  \State gather $b_j$ for refined keys and $v_j$ for live keys
  \State $c^{ba},c^{bb}\gets K^{b}_t\,[Q^{a};Q^{b}]^\top$ \Comment{second INT8 WGMMA}
  \State $s_{ij}\gets\tilde s_{ij}+[\mathrm{refine}_j]\,\eta_ie_j\,(c^{ba}_{ji}+c^{bb}_{ji}/256)/256$
  \State $p_{ij}\gets[\mathrm{live}_j\wedge s_{ij}\ge-T]\;2^{s_{ij}}$ \Comment{final weight, FP32}
  \State add the weights of cut keys to the tile's virtual row \Comment{finite depth only}
  \State $P^{hi}\gets\operatorname{bf16}(p)$;\quad $P^{lo}\gets\operatorname{bf16}(p-P^{hi})$
  \State $A^\top\gets A^\top+V_t^\top P^{hi\top}+V_t^\top P^{lo\top}$ \Comment{BF16 WGMMA, channels on $M$}
  \State $L_i\gets L_i+\sum_jp_{ij}$
\EndFor
\State write $(A,L)$ to this split's slot \Comment{combine: $o_i=\sum_{\text{slots}}A_i\big/\sum_{\text{slots}}L_i$ in slot order}
\end{algorithmic}
\end{algorithm}

Algorithm~\ref{alg:backward} gives the backward for one work-list tile, a key
block $n$ of one head. A preprocess has computed $D_i=dO_i^\top o_i$, the
base-2 log-sum-exp $\ell_i$, and the grid exponent $\sigma=\log_2\alpha$ of
the tile's request and KV head from Equation~\ref{eq:dq-bound}; $c$ is the
softmax scale.

\begin{algorithm}[H]
\caption{FoldAttention backward, one key block}
\label{alg:backward}
\small
\begin{algorithmic}[1]
\Require $Q,dO\in\mathbb R^{S\times D}$, $K_n,V_n$, $\ell$, $D$, grid exponent $\sigma$
\State load $K_n,V_n$; $dK_n,dV_n\gets0$ \Comment{FP32 registers}
\For{each query block $m$ in the causal range of $n$}
  \State $S\gets Q_mK_n^\top$;\quad $\tilde P\gets\exp_2\bigl(c\log_2(e)\,S-\ell_m+\sigma\bigr)$ \Comment{$\tilde P=2^\sigma P$}
  \State $dV_n\gets dV_n+\tilde P^\top dO_m$
  \State $dP\gets dO_mV_n^\top$;\quad $d\tilde S\gets\tilde P\circ(dP-D_m)$
  \State $\widehat{dQ}\gets\operatorname{RNE}_{\mathrm{s32}}(d\tilde S\,K_n)$ \Comment{already on the grid}
  \State bulk integer add of $\widehat{dQ}$ into $dQ^{\mathrm{acc}}_m$ \Comment{order-free}
  \State $dK_n\gets dK_n+d\tilde S^\top Q_m$
\EndFor
\If{this CTA owns the key block's whole group}
  \State store $c\,2^{-\sigma}dK_n$ and $2^{-\sigma}dV_n$
\ElsIf{the group is split into subgroups}
  \State write the FP32 partial; the last subgroup to arrive sums all partials in subgroup order and stores
\Else
  \State round $dK_n,dV_n$ onto their integer grids and bulk-add them
\EndIf
\State postprocess: $dQ\gets c\,2^{-\sigma}dQ^{\mathrm{acc}}$, and likewise for integer $dK,dV$
\end{algorithmic}
\end{algorithm}

\clearpage
\section{Experimental details}
\label{app:protocol}

\paragraph{Baseline configurations.}
We use FlashAttention-3 3.0.0, FlashAttention-4 4.0.0b31, FlashInfer 0.7.0,
cuDNN 9.26, DASH at commit d87bcc9, vLLM 0.30.0, PAT at commit 8cb067f, and
sglang-kernel 0.4.7. Each library runs every configuration it offers at a
shape: paged and
contiguous caches, packed GQA, split counts 1--16 for FlashAttention-3, the
tensor-core and CUDA-core FlashInfer backends, and cuDNN through FlashInfer's
paged binding and through PyTorch SDPA. Configurations that fail are
dropped, the rest are timed over nine rounds, and each library is
represented by its fastest.

\paragraph{Timing.}
The selected baselines and FoldAttention are then timed together by CUDA-graph
replay, with the L2 cache evicted by reading a 256\,MB buffer before every
sample and the kernel order rotated with a stride coprime to the number of
kernels. We run at least 21 rounds, drop the first, and report medians of
per-round paired ratios.

\paragraph{Accuracy reference.}
Decode outputs are scored against FP32 attention over the same BF16 inputs,
and backward gradients against FP32 gradients computed one head at a time
without TF32. A kernel whose error exceeds 20 times the best error at its
precision is treated as computing a different function and excluded from
every ratio.

\paragraph{DASH.}
DASH and FlashAttention-3 register the same PyTorch operators and cannot be
loaded together. DASH therefore runs in its own process with FlashAttention-4 and
FoldAttention, and its times are joined to the FlashAttention-3 process
through FoldAttention's time in each. FlashAttention-4's ratio to
FoldAttention agrees between the two processes to within 1\% at 59 of 68
shapes and 2.5\% at every shape, which bounds the join's error. DASH was built
from commit d87bcc9
with a one-line patch restoring FlashAttention-3's masking bound for the
diagonal block of its reversed causal loop; without it, DASH's head-dim-64
causal gradients have 15--60$\times$ the error of every other kernel.
DASH has no variable-length schedule and is measured on dense batches only.

\paragraph{Shared-prefix baselines.}
vLLM's cascade path runs FA-3 over the prefix the whole batch holds, FA-3 over
each request's remainder, and a state merge, so a prefix tree cascades over
its root. PAT launches on the legacy default stream, which a CUDA graph
cannot capture, and is timed eagerly with the same L2 eviction. At $D=64$
with $G=8$, its output fails the accuracy criterion even on random inputs,
and one run ended in an illegal memory access. We therefore report no
gpt-oss-20b cells. On Qwen3-30B-A3B,
the accuracy rule drops it in one cascade cell, and in two others its error is
7--8 times the BF16 kernels'; without those two, its cascade range is
1.18--1.48$\times$. FastTree fails the accuracy criterion on three
single-prefix cascades, which are omitted.

\clearpage
\section{Additional decode results}
\label{app:decode}

\subsection{Latency against error on every generation}
Figure~\ref{fig:pareto} plots latency against error for every generation of
Table~\ref{tab:decode}, along the depth dial.

\begin{figure}[H]
\centering
\includegraphics[width=5.5in]{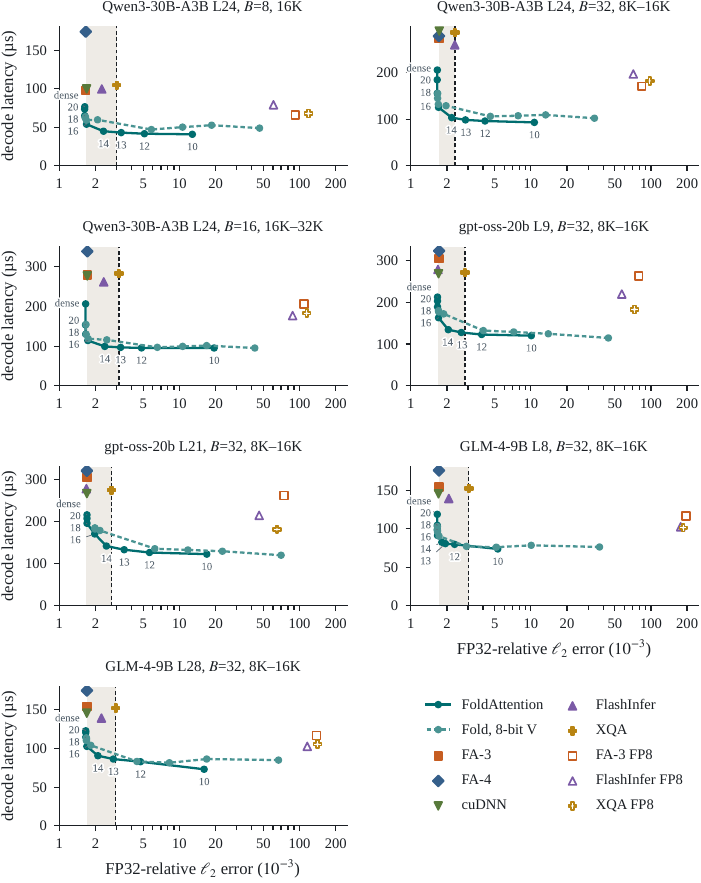}
\caption{Decode latency against FP32-relative error on the seven generations
of Table~\ref{tab:decode}. FoldAttention is a curve over depth $T$, BF16
(solid) or 8-bit (dashed) values, with each point's $T$ labeled on the BF16
curve; the 8-bit curve's points are the same depths in the same order.
Libraries are points, FP8 caches hollow.
The shaded band spans the BF16 kernels' errors. On gpt-oss-20b layer 21, no
FoldAttention variant meets the most accurate BF16 kernel's error.}
\label{fig:pareto}
\end{figure}

\clearpage
\subsection{Context sweeps}
Figure~\ref{fig:sweep-full} extends the sweep of Figure~\ref{fig:sweep} to
every measured group size and adds depth 14.

\begin{figure}[H]
\centering
\includegraphics[width=5.5in]{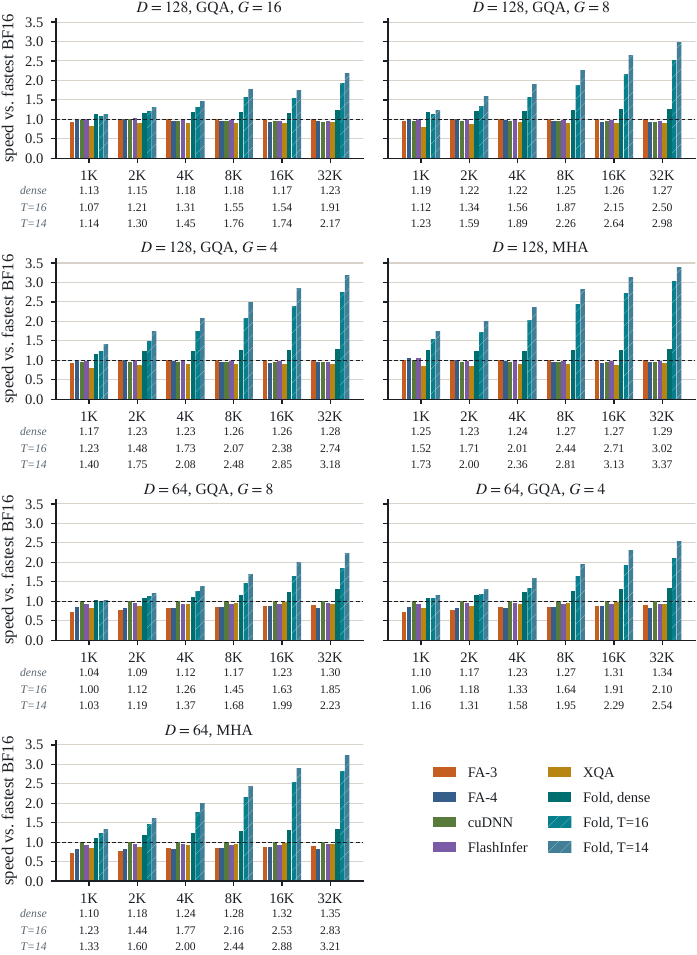}
\caption{The sweep of Figure~\ref{fig:sweep} for every measured group size
(ragged batches of 256 KV heads). At $D=64$, groups of one and four run on
the 128-key tile.}
\label{fig:sweep-full}
\end{figure}

\clearpage
Figure~\ref{fig:sweep-uniform} repeats the context sweep
on uniform batches, where FlashInfer or cuDNN is usually the fastest
baseline and dense FoldAttention is 0.95$\times$ at $D=64$, $G=8$ and 1K
context and 1.03--1.31$\times$ elsewhere.

\begin{figure}[H]
\centering
\includegraphics[width=5.5in]{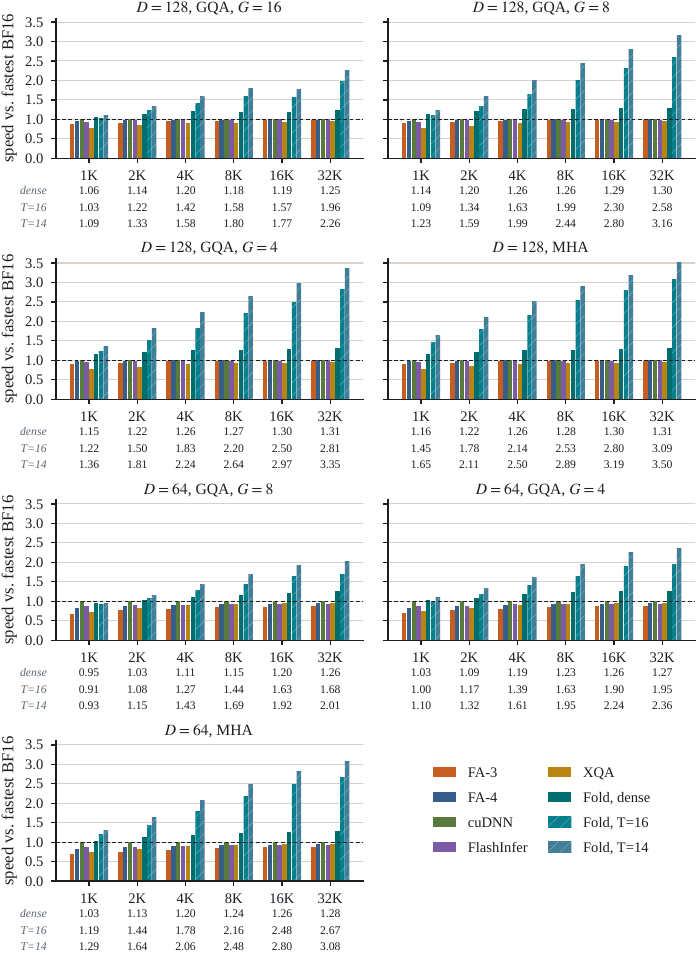}
\caption{The sweep of Figure~\ref{fig:sweep-full} on uniform batches.}
\label{fig:sweep-uniform}
\end{figure}

\clearpage
Figure~\ref{fig:sweep-v8} stores values in two
E4M3 planes, which only the 64-key tile reads; dense decode with 8-bit values
reaches 1.59--1.86$\times$ at 32K.

\begin{figure}[H]
\centering
\includegraphics[width=5.5in]{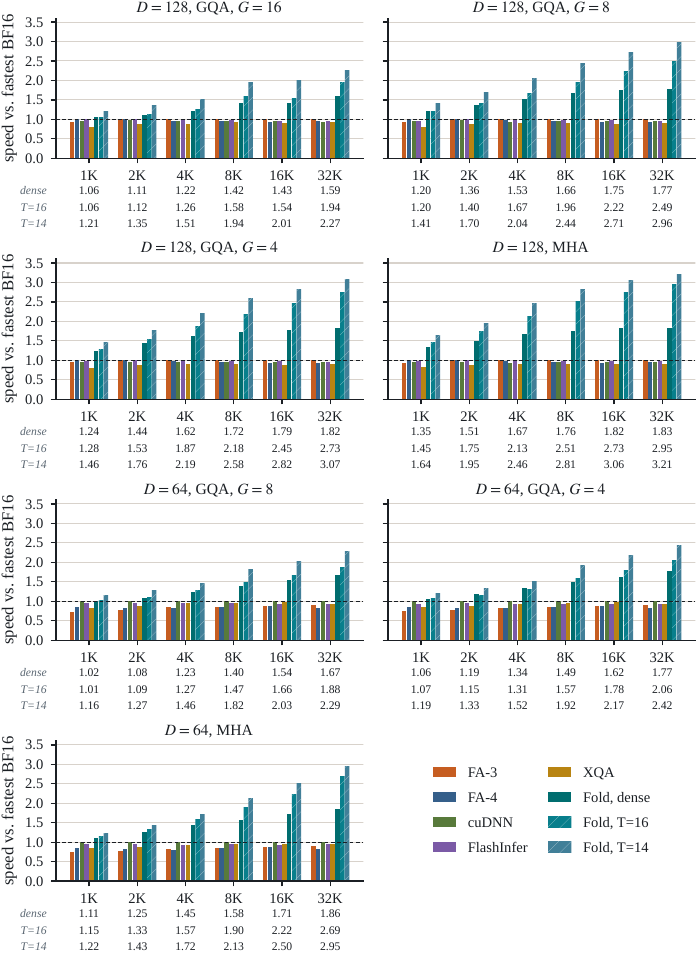}
\caption{The sweep of Figure~\ref{fig:sweep-full} with 8-bit values. ``Dense''
reads every value row in two E4M3 planes; baselines read BF16 caches.}
\label{fig:sweep-v8}
\end{figure}

\subsection{Reference quality}
Figure~\ref{fig:headroom} shows how far the estimate of
Section~\ref{sec:contract} lands from the truth over every row of the seven
generations: the true log-sum-exp sits $-0.4$ to $17.2$ binades above $Z_i$,
1.0 at the median. Decoding Qwen3-8B for real at 64K and 128K context with
YaRN (Section~\ref{sec:eval-e2e}), it sits $-0.9$ to $28.2$ binades above
$Z_i$. Over all 1.8 billion rows decoded in Qwen3-8B at 8K--128K, including
LongBench, 118 fall below the certificate's window $[2^{-1},2^{100}]$, each
with $Z_i$ 1.0--1.6 binades above the log-sum-exp; none exceeds it.

A better estimate would buy little. We drive one kernel-level call with two references: the mass estimate of
Section~\ref{sec:contract}, and each row's exact log-sum-exp of its FP32
scores. Refine gates, weight terms, the cut model (the running mean value)
and depths 10--20 are the same for both; only $Z_i$ differs. Cells are the
four of Table~\ref{tab:ablation} at 4K and 16K on uniform batches of 256 KV
heads. Because the estimate sits below the log-sum-exp, a given depth keeps
more keys under it and has lower error, so we compare at matched error: for
each point of the estimate we interpolate the oracle's bytes per key, linear
in log error between adjacent depths. At depths 12--14 the oracle needs
0.94--0.99 of the estimate's bytes (0.97 by geometric mean), with its largest
saving on GLM-4-9B at 16K. In dense decode the two read within 3\% of each
other's bytes and within 1.1\% of each other's error.

\subsection{Tile layouts}
\label{app:tiles}
At $D=64$, a group of up to four rows with BF16 values leaves most of the
product's $N$ idle, while each tile pays its copies, round trips, and barriers
for 64 keys. These groups walk 128-key tiles: each row of $M$ holds two keys
of plane A as stored, and each query takes two columns, $[q\,|\,0]$ and
$[0\,|\,q]$, in the idle part of $N$. The products are those of two 64-key
tiles, so the logits are the same bits, and each copy, round trip, and
barrier serves 128 keys. At $D=128$ the wider tile halves the CTAs an SM
holds and gains nothing.

A cascade level stacks the query rows of the requests that hold it, and up to
64 rows fit the decode kernel's $N$. Past 64 they fill $M$ on their own, and
the level runs on a kernel laid out as prefill is, with 64 query rows and 64
keys per tile and two consumer warpgroups taking turns at the tensor cores.
Its logit is one FP16 product of $\operatorname{fp16}(a_j+b_j/256)$, rounded
once per prefix, against the rows' queries; 11 bits put each logit within
about $2^{-11}$ of $|q||k|$, below the BF16 rounding of the weight. Draft
batches past 64 rows run on the same kernel.

\subsection{Cache construction}
Writing a prompt into the dense cache takes 0.18--0.25\,ms per layer
(0.6--11\% of FlashAttention-4's prompt attention) against 0.05--0.09\,ms
for a BF16 page append. A finite depth with BF16 values also needs the block
model, which the host fits in 6.6--11.3\,ms per layer per prefill, 0.2 to 5.3
times the prompt's attention from 32K down to 2K tokens. With 8-bit values the
cut model is a running mean, and a finite depth costs nothing at prefill.

\subsection{Sparse and quantized decode at matched error}
\label{app:methods}
Figure~\ref{fig:methods} places sparse and quantized decode on the final
state of each generation at one byte model. Each method is emulated in exact
arithmetic, and every byte it reads is counted, metadata included: Quest's
per-page key bounds, the block scales and residuals of Faster Flash
Decoding, KIVI's group scales and its recent tokens kept in BF16. Quest keeps
the sink and last page and takes the top pages to a token budget, shared by
the group or per head; Faster Flash Decoding keeps sub-blocks whose screened
score reaches the pseudo-maximum of the first and last blocks less $\delta$,
and drops the others' mass. FoldAttention's points are its kernel's own
counts. The same emulation gives the running-maximum gate of
Section~\ref{sec:eval-ablation}: each split's maximum is taken over its
64-key tiles up to the current one, and the kernel's refine and depth gates
are applied against it with the same coarse logits.

\begin{figure}[H]
\centering
\includegraphics[width=5.5in]{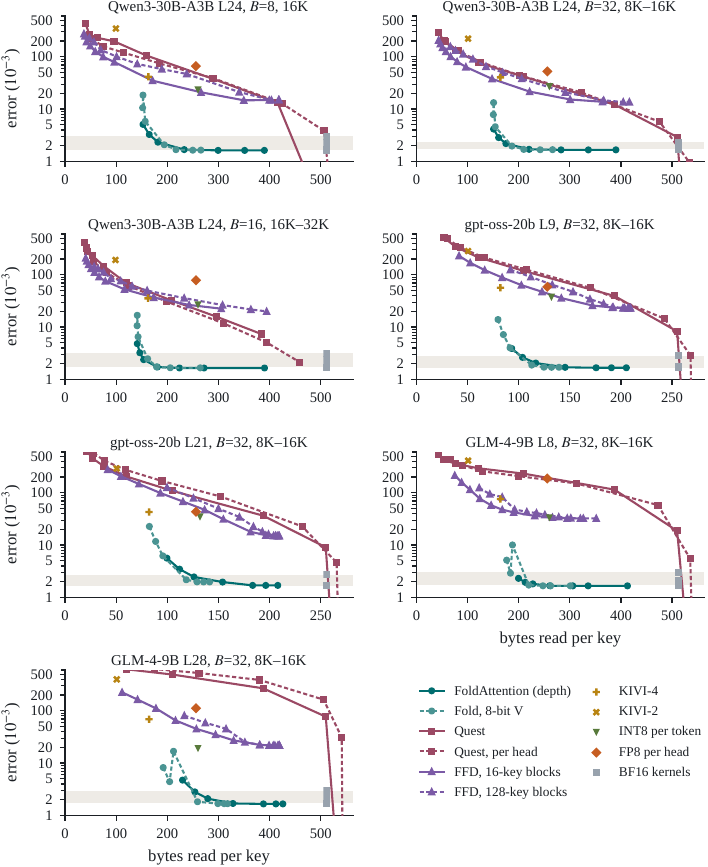}
\caption{FP32-relative error against bytes read per key on the final state
of each generation, every method emulated in exact arithmetic with its
metadata counted. The shaded band spans the BF16 kernels' errors, which read
$4D$ bytes per key. FoldAttention is a curve over depth $T$, from dense at
the right; Quest and Faster Flash Decoding are curves over their budgets.}
\label{fig:methods}
\end{figure}

\subsection{Cascades and drafts}
\label{app:compose}
The cascade takes its references from the flat decode's estimate, because the
estimate's prepass sees only the level it runs on. Charging the flat decode's
whole front to it, while the baselines' times exclude their own appends,
gives 0.90--1.23$\times$ (1.07$\times$) over the fastest kernel on cascades
and 0.79--1.13$\times$ (0.98$\times$) on trees. Table~\ref{tab:compose}
gives the kernel-level speedups over each prefix-sharing kernel.
FlashInfer's own paged decode is 2.3--7.9$\times$
faster than its cascade wrapper at these sizes. On draft verification (both
models, 8 or 32 requests, a 4K or 16K cache), dense FoldAttention is
1.07--1.24$\times$ faster than the fastest library for chains of two drafts,
0.99--1.17$\times$ for four, 0.58--0.84$\times$ for eight, and
0.43--0.74$\times$ for sixteen, and 0.71--1.09$\times$ and 0.72--1.14$\times$
against SGLang's tree verification for trees of eight and sixteen. The drafts
share every byte the kernel reads, but each adds a column to every product, so
the work grows with the number of drafts while the traffic does not.

\begin{table}[H]
\centering
\caption{Cascade speedup of dense FoldAttention over each kernel: the range
of the kernel's time over ours, with the geometric mean in parentheses.
``Fastest'' takes the fastest kernel of any kind in each cell. PAT has no
gpt-oss-20b cells and FastTree no result on three cascades
(Appendix~\ref{app:protocol}).}
\label{tab:compose}
\scriptsize
\setlength{\tabcolsep}{4pt}
\begin{tabular}{lccccc}
\toprule
Set & FlashInfer cascade & vLLM cascade & PAT & FastTree & Fastest \\
\midrule
Cascades (16) & 3.93--8.88 (6.18) & 1.08--1.76 (1.34) & 1.13--1.48 (1.32) & 1.23--1.95 (1.61) & 1.04--1.42 (1.22) \\
Prefix trees (8) & 3.38--8.88 (5.27) & 1.48--2.88 (1.80) & 0.90--1.29 (1.07) & 1.02--1.71 (1.24) & 0.90--1.33 (1.09) \\
\bottomrule
\end{tabular}
\end{table}

\clearpage
\section{Additional backward results}
\label{app:backward}

Figures~\ref{fig:backward-causal} and~\ref{fig:backward-full} give the grid
FlashAttention-3 and DASH report: MHA and GQA with groups of eight, with and
without a causal mask. Table~\ref{tab:bwd-shapes} gives every model
shape of Table~\ref{tab:backward}.

\begin{figure}[H]
\centering
\includegraphics[width=5.5in]{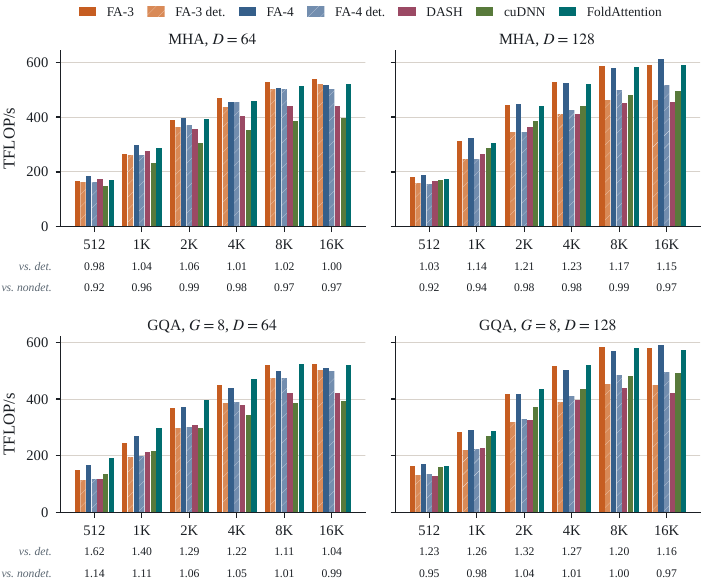}
\caption{Backward throughput with a causal mask on the grid FlashAttention-3
and DASH report (16K tokens per batch, model width 2048), for MHA and for GQA
with eight query heads per KV head. Hatched bars are deterministic modes. The
table under each panel is FoldAttention's speedup over the fastest
deterministic and the fastest nondeterministic kernel.}
\label{fig:backward-causal}
\end{figure}

\begin{figure}[H]
\centering
\includegraphics[width=5.5in]{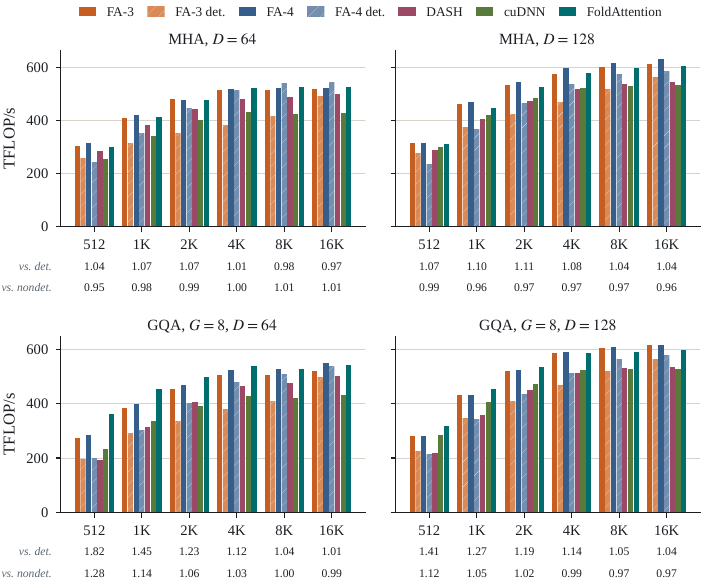}
\caption{The grid of Figure~\ref{fig:backward-causal} without a mask.}
\label{fig:backward-full}
\end{figure}

\begin{table}[H]
\centering
\caption{Backward time of FoldAttention and each baseline's time over it.
$H/H_{KV}$ gives query and KV heads. Shapes marked $\dagger$ are below 16K
tokens per batch and are excluded from the headline geometric mean.}
\label{tab:bwd-shapes}
\scriptsize
\begin{tabular}{rrrrlrrrrrrr}
\toprule
$B$ & $H/H_{KV}$ & $S$ & $D$ & mask & ours ($\mu$s) & FA-3 & FA-3 det. & FA-4 & FA-4 det. & cuDNN & DASH \\
\midrule
2 & 32/8 & 8192 & 128 & causal & 4942 & 1.03 & 1.29 & 1.04 & 1.19 & 1.21 & 1.37 \\
4 & 32/8 & 4096 & 128 & causal & 2599 & 1.05 & 1.35 & 1.05 & 1.28 & 1.20 & 1.37 \\
8 & 32/8 & 2048 & 128 & causal & 1548 & 1.10 & 1.38 & 1.08 & 1.34 & 1.18 & 1.39 \\
16 & 32/8 & 1024 & 128 & causal & 1056 & 1.21 & 1.43 & 1.16 & 1.39 & 1.18 & 1.41 \\
2 & 32/32 & 8192 & 128 & causal & 4865 & 1.00 & 1.25 & 0.99 & 1.14 & 1.19 & 1.30 \\
2 & 32/8 & 8192 & 128 & full & 9026 & 1.03 & 1.17 & 1.00 & 1.07 & 1.16 & 1.16 \\
2 & 64/8 & 8192 & 64 & causal & 5438 & 1.02 & 1.11 & 1.07 & 1.11 & 1.36 & 1.24 \\
8 & 64/8 & 2048 & 64 & causal & 1660 & 1.12 & 1.39 & 1.09 & 1.37 & 1.38 & 1.32 \\
1 & 8/2 & 8192 & 128 & causal$^\dagger$ & 571 & 1.02 & 1.47 & 1.17 & 1.35 & 1.36 & 1.36 \\
1 & 8/2 & 2048 & 128 & causal$^\dagger$ & 83 & 1.14 & 1.38 & 1.08 & 1.28 & 1.27 & 1.58 \\
1 & 16/2 & 4096 & 64 & causal$^\dagger$ & 197 & 0.99 & 1.30 & 1.16 & 1.30 & 1.51 & 1.24 \\
\bottomrule
\end{tabular}
\end{table}

\begin{table}[H]
\centering
\caption{Per-element gradient error against FP64 on attention operands
captured from the training run of Section~\ref{sec:eval-e2e} (a
Llama-3.2-1B-shaped model, $H/H_{KV}=32/8$, $D=64$, 8192 tokens), for
FoldAttention and FlashAttention-3. $\ell_2$ is the $dQ$ error in $10^{-3}$;
p99 is the 99th percentile of per-element relative error. The binade columns
give FoldAttention's median relative error over FlashAttention-3's among
$dQ$ elements that many binades below the largest of their request and KV
head; the last columns cover elements $2^{-20}$ or more below it. $dK$ and
$dV$ match FlashAttention-3's $\ell_2$ error to three digits in every
layer.}
\label{tab:grad-elements}
\scriptsize
\setlength{\tabcolsep}{4pt}
\begin{tabular}{lcccccc}
\toprule
Step, layer & $\ell_2$ Fold / FA-3 & p99 Fold / FA-3 & 0 to $-13$ & $-14$ to $-19$ & share $\le-20$ & median $\le-20$, Fold / FA-3 \\
\midrule
500, 0 & 2.99 / 2.99 & 1.45 / 1.52 & 1.00--1.03 & 1.07--1.12 & 2.7\% & 0.58 / 0.46 \\
500, 7 & 5.01 / 5.01 & 12.5 / 19.0 & 1.00 & 1.00 & 6.3\% & 1.12 / 2.20 \\
500, 15 & 5.01 / 5.01 & 12.0 / 14.8 & 1.00 & 1.00 & 5.7\% & 1.30 / 2.03 \\
1999, 0 & 3.36 / 3.36 & 1.23 / 1.25 & 1.00--1.04 & 1.08--1.21 & 2.1\% & 0.40 / 0.31 \\
1999, 7 & 5.67 / 5.67 & 3.17 / 3.27 & 1.00 & 1.00 & 0.9\% & 2.26 / 2.67 \\
1999, 15 & 4.43 / 4.44 & 1.69 / 1.54 & 1.00 & 1.00--2.70 & 2.2\% & 0.46 / 0.0047 \\
\bottomrule
\end{tabular}
\end{table}

\begin{figure}[H]
\centering
\includegraphics[width=5.5in]{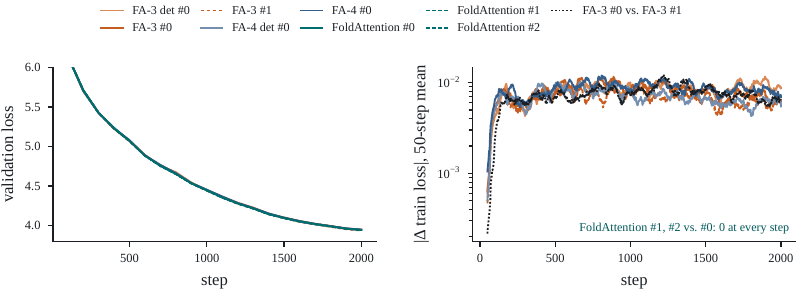}
\caption{A 1B Llama trained from scratch on WikiText-103 for 2000 steps of
16K tokens with each library's attention, from the same initialization and
data order. Left: validation loss. Right: each run's training loss against
FoldAttention's first run, as a 50-step mean of the absolute difference.
FoldAttention's second and third runs match its first at every step and have
no point on the log scale; the dotted line is FlashAttention-3 against itself.}
\label{fig:train-curve}
\end{figure}

\clearpage
\section{Determinism}
\label{app:determinism}

As DASH and TBIK do~\citep{qiang2026dash,zhang2025tbik},
Table~\ref{tab:determinism} measures the largest change in any gradient
element of one request when it is rerun, batched, or packed.

\begin{table}[H]
\centering
\caption{Largest absolute difference in any $dQ,dK,dV$ element of one
request against its first result, over three dense shapes and two packed
ones (0 means every bit matched). ``Repeatable'' counts backward benchmark shapes
with identical bits over three calls; error is the worst FP32-relative
$\ell_2$ over $dQ,dK,dV$ in $10^{-3}$.}
\label{tab:determinism}
\scriptsize
\begin{tabular}{lccccc}
\toprule
Kernel & Ten runs & Alone vs.\ batch of 4 & Alone vs.\ packed & Repeatable & Worst error \\
\midrule
FA-3 & $3.9\times10^{-3}$ & $2.0\times10^{-3}$ & $2.0\times10^{-3}$ & 1/73 & 2.53 \\
FA-4 & $2.0\times10^{-3}$ & $4.9\times10^{-4}$ & $2.4\times10^{-4}$ & 3/141 & 2.53 \\
cuDNN & $4.9\times10^{-4}$ & $4.9\times10^{-4}$ & -- & 11/136 & 3.02 \\
FA-3 det. & 0 & 0 & 0 & 73/73 & 2.53 \\
FA-4 det. & 0 & 0 & 0 & 141/141 & 2.53 \\
DASH & 0 & 0 & -- & 68/68 & 2.52 \\
FoldAttention & 0 & 0 & 0 & 141/141 & 2.53 \\
\bottomrule
\end{tabular}
\end{table}

Table~\ref{tab:decode-inv} measures decode: the same request run twice, and
run alone instead of inside a ragged batch. A fixed number of keys per split
also makes decode invariant to the split count; with the best such chunk per
model, the generations' steps are 2\% faster to 6\% slower than at the
default split.

\begin{table}[H]
\centering
\caption{Largest absolute output difference for the first request of a
ragged batch of 32 at 16K context, run twice, and run alone instead of in
the batch. FoldAttention's default split is chosen from the batch (8 splits
in the batch, 64 alone); a fixed split of 4 makes it batch invariant.
Changing the fixed split from 4 to 8 changes the combine's grouping and the
bits. cuDNN 9.26's paged decode is not repeatable on the $D=128$ batch.}
\label{tab:decode-inv}
\scriptsize
\begin{tabular}{lcccc}
\toprule
& \multicolumn{2}{c}{Qwen3-30B, $D=128$} & \multicolumn{2}{c}{gpt-oss-20b, $D=64$} \\
Kernel & repeat & alone vs.\ batch & repeat & alone vs.\ batch \\
\midrule
FA-3 & 0 & $4.9\times10^{-4}$ & 0 & $1.6\times10^{-2}$ \\
FA-4 & 0 & 0 & 0 & 0 \\
FlashInfer & 0 & $2.0\times10^{-3}$ & 0 & $1.6\times10^{-2}$ \\
XQA & 0 & $2.0\times10^{-3}$ & 0 & $3.1\times10^{-2}$ \\
cuDNN & $2.4\times10^{-4}$ & $9.8\times10^{-4}$ & 0 & $1.6\times10^{-2}$ \\
FoldAttention, default split & 0 & $1.5\times10^{-8}$ & 0 & $7.6\times10^{-6}$ \\
FoldAttention, fixed split & 0 & 0 & 0 & 0 \\
\bottomrule
\end{tabular}
\end{table}

\end{document}